\documentclass{article}

\usepackage{microtype}
\usepackage{graphicx}
\usepackage{booktabs} 

\usepackage{hyperref}

\usepackage{algorithm}
\usepackage{algpseudocode}
\usepackage{multirow}
\usepackage{enumitem}
\usepackage{caption}
\usepackage{subfig}
\usepackage{wrapfig}
\usepackage[table]{xcolor}
\usepackage{colortbl}
\usepackage{pifont}

\definecolor{mygray}{gray}{.9}
\newcommand{\stdv}[1]{\scalebox{.70}{~$\pm$~#1}}

\def\ie{\emph{i.e.,}}

\usepackage[accepted]{icml2024}

\usepackage{amsmath}
\usepackage{amssymb}
\usepackage{mathtools}
\usepackage{amsthm}

\usepackage[capitalize,noabbrev]{cleveref}

\theoremstyle{plain}
\newtheorem{theorem}{Theorem}[section]
\newtheorem*{theoremm}{Theorem}

\theoremstyle{definition}

\newtheorem{assumption}[theorem]{Assumption}
\theoremstyle{remark}

\usepackage[textsize=tiny]{todonotes}

\newcommand{\state}{\mathbf{s}}

\newcommand{\action}{\mathbf{a}}

\newcommand{\policy}{\pi}

\newcommand{\reward}{r}
\newcommand{\rewmodel}{\widehat{\reward}_\psi}

\DeclareMathOperator*{\expec}{\mathbb{E}}

\newcommand{\KL}{D_{\mathrm{KL}}}

\icmltitlerunning{RIME: Robust Preference-based Reinforcement Learning with Noisy Preferences}

\begin{document}

\twocolumn[
\icmltitle{RIME: Robust Preference-based Reinforcement Learning with Noisy Preferences}

\icmlsetsymbol{equal}{*}

\begin{icmlauthorlist}
\icmlauthor{Jie Cheng}{1,2}
\icmlauthor{Gang Xiong}{1,2}
\icmlauthor{Xingyuan Dai}{1,2}
\icmlauthor{Qinghai Miao}{2}
\icmlauthor{Yisheng Lv}{1,2}
\icmlauthor{Fei-Yue Wang}{1,2}
\end{icmlauthorlist}

\icmlaffiliation{1}{State Key Laboratory of Multimodal Artificial Intelligence Systems, CASIA}
\icmlaffiliation{2}{School of Artificial Intelligence, the University of Chinese Academy of Sciences}

\icmlcorrespondingauthor{Yisheng Lv}{yisheng.lv@ia.ac.cn}

\icmlkeywords{preference-based reinforcement learning, human-in-the-loop reinforcement learning, deep reinforcement learning}

\vskip 0.3in
]

\printAffiliationsAndNotice{}  

\begin{abstract}
Preference-based Reinforcement Learning (PbRL) circumvents the need for reward engineering by harnessing human preferences as the reward signal. However, current PbRL methods excessively depend on high-quality feedback from domain experts, which results in a lack of robustness. In this paper, we present RIME, a robust PbRL algorithm for effective reward learning from noisy preferences. Our method utilizes a sample selection-based discriminator to dynamically filter out noise and ensure robust training. To counteract the cumulative error stemming from incorrect selection, we suggest a warm start for the reward model, which additionally bridges the performance gap during the transition from pre-training to online training in PbRL. Our experiments on robotic manipulation and locomotion tasks demonstrate that RIME significantly enhances the robustness of the state-of-the-art PbRL method. Code is available at \url{https://github.com/CJReinforce/RIME_ICML2024}.
\end{abstract}

\section{Introduction}
Reinforcement Learning (RL) has demonstrated remarkable performance in various domains, including gameplay~\cite{perolat2022mastering,kaufmann2023champion}, robotics~\cite{chen2022towards}, autonomous systems~\cite{bellemare2020autonomous,zhou2020smarts}, multimodal~\cite{yue2024sc}, etc. The success of RL frequently relies on the meticulous crafting of reward functions, a process that can be both time-consuming and susceptible to errors. In this context, Preference-Based RL (PbRL)~\cite{akrour2011preference,cheng2011preference,christiano2017deep} emerges as a valuable alternative, eliminating the requirement for manually designed reward functions. PbRL adopts a human-in-the-loop paradigm, where human teachers provide preferences over distinct agent behaviors as the reward signal. 

Nevertheless, existing works in PbRL have primarily focused on enhancing feedback-efficiency, aiming to maximize the expected return with few feedback queries. This focus induces a substantial reliance on high-quality feedback, typically assuming expertise on the teachers~\cite{liu2022meta,kim2022preference}. However, humans are prone to errors~\cite{christiano2017deep}. In broader applications, feedback is often sourced from non-expert users or crowd-sourcing platforms, where the quality can be inconsistent and noisy. Further complicating the matter, \citet{lee2021b} showed that even a mere 10\% corruption rate in preference labels can dramatically degrade the performance. Therefore, the lack of robustness to noisy preference labels hinders the wide applicability of PbRL.

Meanwhile, learning from noisy labels, also known as robust training, is a rising concern in deep learning, since such labels severely degrade the generalization performance of deep neural networks. \citet{song2022learning} classifies robust training methods into four key categories: robust architecture~\cite{cheng2020weakly}, robust regularization~\cite{xia2020robust}, robust loss design~\cite{lyu2019curriculum}, and sample selection~\cite{li2020gradient,song2021robust}. However, it is challenging to effectively incorporate these advanced methods for robust training in PbRL. This complexity arises from the pursuit of feedback-efficiency and cost reduction, necessitating access to a limited amount of feedback. Simultaneously, the distribution shift problem during RL training undermines the assumption of i.i.d input data, a core principle that supports robust training methods in deep learning.

\begin{figure*}[htbp]
\centering
\includegraphics[width=0.9\textwidth]{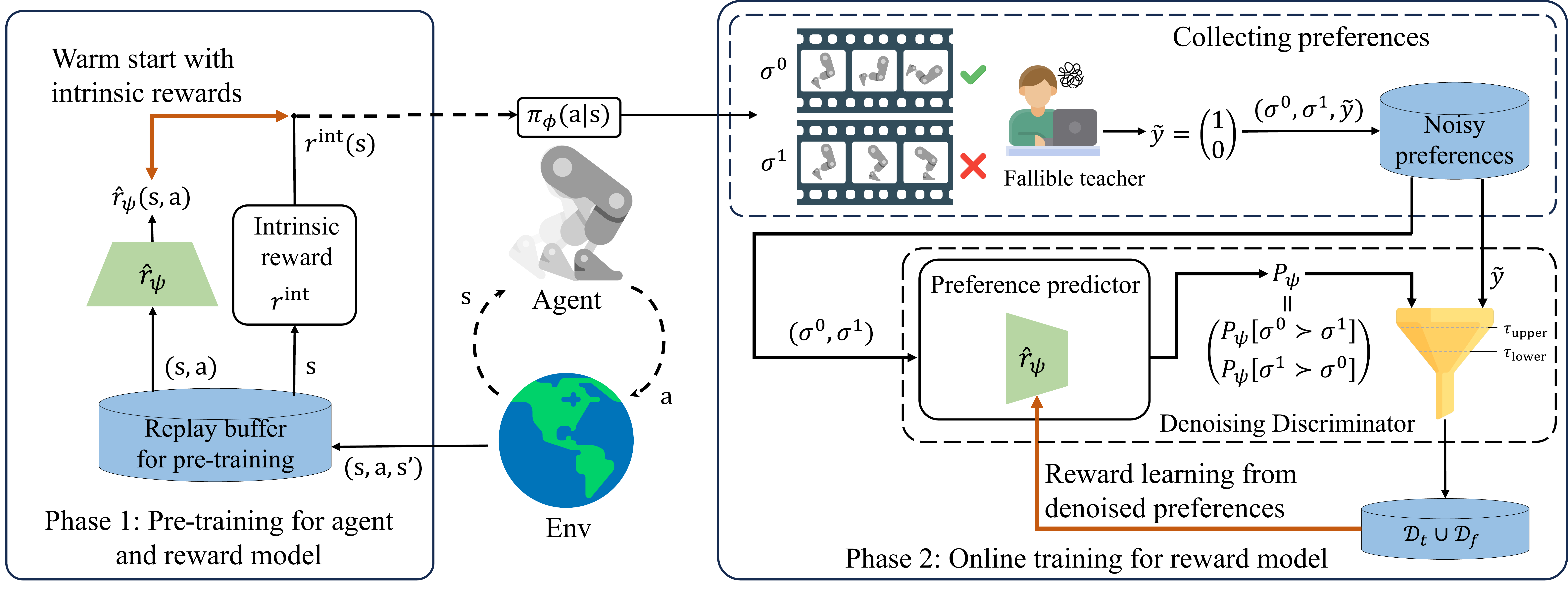}
\caption{Overview of RIME. In the pre-training phase, we warm start the reward model $\hat{r}_\psi$ with intrinsic rewards $r^\text{int}$ to facilitate a smooth transition to the online training phase. Post pre-training, the policy, Q-network, and reward model $\hat{r}_\psi$ are all inherited as initial configurations for online training. During online training, we utilize a denoising discriminator to screen denoised preferences for robust reward learning. This discriminator employs a dynamic lower bound $\tau_\text{lower}$ on the KL divergence between predicted preferences $P_\psi$ and annotated preference labels $\tilde{y}$ to filter trustworthy samples $\mathcal{D}_t$, and an upper bound $\tau_\text{upper}$ to flip highly unreliable labels $\mathcal{D}_f$.}
\label{fig:overview}
\end{figure*}

In this work, we aim to improve the robustness of preference-based RL methods on noisy and quantitatively limited preferences. To this end, we present RIME: \textbf{R}obust preference-based re\textbf{I}nforcement learning via war\textbf{M}-start d\textbf{E}noising discriminator. RIME modifies the training paradigm of the reward model in the widely-adopted two-phase (\ie pre-training and online training phases) pipeline of PbRL. Figure \ref{fig:overview} shows an overview of RIME. In particular, to empower robust learning from noisy preferences, we introduce a denoising discriminator. It utilizes dynamic lower and predefined upper bounds on the Kullback–Leibler (KL) divergence between predicted and annotated preference labels to filter samples. Further, to mitigate the accumulated error caused by incorrect filtration, we propose to warm start the reward model during the pre-training phase for a good initialization. Moreover, we find that the warm start also bridges the performance gap that occurs during the transition from pre-training to online training. Our experimental results indicate that RIME significantly outperforms existing baselines under noisy preference conditions, thereby substantially enhancing robustness for PbRL.

In summary, our work has three main contributions. First, we present RIME, a robust reward learning algorithm for PbRL, designed to effectively train reward models from noisy feedback—an important and realistic topic that has not been studied extensively. Second, we observe a dramatic performance gap during the transition from pre-training to online training in PbRL and propose to warm start the reward model for a seamless transition, which proves to be crucial for both robustness and feedback-efficiency in limited-feedback cases. Last, we show that RIME outperforms existing PbRL baselines under noisy feedback settings, across a diverse set of robotic manipulation tasks from Meta-World~\cite{yu2020meta} and locomotion tasks from the DeepMind Control Suite~\cite{tassa2018deepmind,tassa2020dm_control}, and further is more suitable for non-expert humans.

\section{Related work}
{\bf Preference-based Reinforcement Learning}.
Incorporating human feedback into the training of reward models has proven effective in various domains, including natural language processing~\citep{ouyang2022training}, multimodal \citep{lee2023aligning}, and reinforcement learning \citep{christiano2017deep, ibarz2018reward, hejna2023few}. In the context of RL, \citet{christiano2017deep} proposed a comprehensive framework for PbRL. To improve feedback-efficiency, PEBBLE \citep{lee2021pebble} used unsupervised exploration for policy pre-training. SURF \citep{park2021surf} employed data augmentation and semi-supervised learning to enrich the preference dataset. RUNE \citep{liang2021reward} encouraged exploration by modulating reward uncertainty. MRN \citep{liu2022meta} introduced a bi-level optimization to optimize the Q-function's performance. PT \citep{kim2022preference} utilized Transformer architecture to model non-Markovian rewards, proving effective in complex tasks.

Despite these advancements, the focus on feedback efficiency should not overshadow the equally critical issue of robustness in PbRL. \citet{lee2021b} indicated that a mere 10\% rate of corrupted preferences can significantly impair algorithmic performance. Furthermore, in more extensive application contexts, the collection of preferences from non-experts increases the likelihood of incorporating incorrect labels. Therefore, enhancing the robustness of PbRL remains a vital research direction. In this work, we address robust PbRL via a warm-start denoising discriminator, which dynamically filters denoised preferences and is more adaptable to cases of distribution shift during RL training.

{\bf Learning from Noisy Labels}.
Learning from noisy labels has gained more attention in supervised learning, particularly in light of the wide presence of noisy or imprecise labels in real-world applications. A variety of approaches have been proposed for robust training~\citep{song2022learning}, including architectural modifications \citep{goldberger2016training}, regularization \citep{lukasik2020does}, loss function designs \citep{zhang2018generalized}, and sample selection methods \citep{wang2021denoising}. Despite these advancements, the direct application of these methods to reward learning in PbRL has presented challenges, mainly due to the limited sample sizes and the absence of i.i.d. of samples. In the context of PbRL, \citet{xue2023reinforcement} proposed an encoder-decoder architecture to model diverse human preferences, which required approximately 100 times the amount of preference labels used in our experiments. Our approach can be situated within the sample selection category and improves robustness while preserving feedback-efficiency.

\textbf{Policy-to-Value Reincarnating RL.} 
Policy-to-value reincarnating RL (PVRL) transfers a sub-optimal teacher policy to a value-based RL student agent~\citep{agarwal2022reincarnating}. \citet{uchendu2023jump} found that a randomly initialized Q-network in PVRL leads to the teacher policy being forgotten quickly. Within the widely-adopted pipeline of PbRL, the challenge intrinsic to PVRL also arises during the transition from pre-training to online training, but has been neglected in previous research~\citep{lee2021pebble, park2021surf, liang2021reward, liu2022meta}. The issue of forgetting the pre-training policy becomes more critical under noisy feedback conditions, as detailed in Section \ref{subsec:pretrain}. Based on the observation, we propose to warm start the reward model for a seamless transition. Our ablation study demonstrates that the warm start is crucial for both robustness and feedback-efficiency.

\section{Preliminaries}
{\bf Preference-based Reinforcement Learning}.
In standard RL, an agent interacts with an environment in discrete time steps~\citep{sutton2018reinforcement}. At each time step $t$, the agent observes the current state $\state_t$ and selects an action $\action_t$ according to its policy $\policy(\cdot|\state_t)$. The environment responds by emitting a reward $\reward(\state_t, \action_t)$ and transitioning to the next state $\state_{t+1}$. The agent's objective is to learn a policy that maximizes the expected return, $\mathcal{R}_0=\sum_{t=0}^\infty\gamma^tr_t$, which is defined as a discounted cumulative sum of the reward with the discount factor $\gamma$.
 
In Preference-based RL, there is no predefined reward function. Instead, a teacher offers preferences between the agent's behaviors, and an estimated reward function $\hat{r}_\psi$ is trained to align with collected preferences. Following previous works~\citep{lee2021pebble,  liu2022meta, kim2022preference}, we consider preferences over two trajectory segments of length $H$, where segment
$\sigma = \{(\state_{1},\action_{1}), ...,(\state_{H}, \action_{H})\}$.
Given a pair of segments $(\sigma^0, \sigma^1)$,
a teacher provides a preference label $\tilde{y}$ from the set $\{(1,0), (0,1), (0.5, 0.5)\}$. The label $\tilde{y}=(1,0)$ signifies $\sigma^0\succ\sigma^1$, $\tilde{y}=(0,1)$ signifies $\sigma^1\succ\sigma^0$, and $\tilde{y}=(0.5, 0.5)$ represents an equally preferable case, where $\sigma^i\succ\sigma^j$ denotes that segment $i$ is preferred over segment $j$.
Each feedback is stored in a dataset $\mathcal{D}$ as a triple $(\sigma^0,\sigma^1,\tilde{y})$. Following the Bradley-Terry model \citep{bradley1952rank}, the preference predicted by the estimated reward function $\hat{r}_\psi$ is formulated as:
\begin{align}
  P_\psi[\sigma^i\succ\sigma^j] = \frac{\exp \left( \sum_t \hat{r}_\psi (\state^i_{t}, \action^i_{t}) \right)}{\sum_{k=i,j}\exp \left(\sum_t \hat{r}_\psi (\state^k_{t}, \action^k_{t}) \right)}
  \label{eq:pref_model}
\end{align}
The estimated reward function $\hat{r}_\psi$ is updated by minimizing the cross-entropy loss between the predicted preferences $P_\psi$ and the annotated labels $\tilde{y}$:
\begin{align}
\mathcal{L}^{\mathtt{CE}}(\psi) = \expec\Big[\mathcal{L}^{\mathtt{Reward}}\Big] = -\expec &\Big[ \tilde{y}(0)\ln P_\psi[\sigma^0\succ\sigma^1] \notag
 \\ + &\tilde{y}(1) \ln P_\psi[\sigma^1\succ\sigma^0]\Big]
\label{eq:CE loss}
\end{align}
The policy $\policy$ can subsequently be updated using any RL algorithm to maximize the expected return with respect to the estimated reward function $\hat{r}_\psi$.

{\bf Unsupervised Pre-training in PbRL}.
Pre-training the agent is important in PbRL because the initial random policy often results in uninstructive preference queries, requiring many queries for even elementary learning progress. Recent studies addressed this issue through unsupervised exploration for policy pre-training~\citep{lee2021pebble}. Specifically, agents are encouraged to traverse a more expansive state space by using an intrinsic reward derived from particle-based state entropy \citep{singh2003nearest}. Formally, the intrinsic reward is defined as \citep{liu2021behavior}:
\begin{align}
    r^{\text{int}}(\state_t)=\log(\lVert\state_t-\state_t^k\rVert)
\end{align}
where $\state_t^k$ is the $k$-th nearest neighbor of $\state_t$. This reward motivates the agent to explore a broader diversity of states. This exploration, in turn, leads to a varied set of agent behaviors, facilitating more informative preference queries.

{\bf Noisy Preferences in PbRL.}
We denote the annotated preference labels as $\tilde{y}$ and the ground-truth preference labels, typically sourced from expert human teachers or scripted teachers, as $y$. To simulate the noise in human annotations, \citet{lee2021b} introduced four noisy 0-1 labeling models: Equal, Skip, Myopic, and Mistake. The ``Mistake" model, in particular, proved to be significantly detrimental to performance across various environments. It hypothesizes that the preference dataset is contaminated with corrupted preferences whose annotated labels $\tilde{y}$ are $1-y$. Drawing on previous insights, our work starts from addressing robust reward learning under the ``Mistake" model settings. This approach is guided by empirical evidence suggesting that solutions developed to overcome complex challenges could be efficiently adapted to simpler cases.

\section{RIME}
In this section, we formally introduce RIME: \textbf{R}obust preference-based re\textbf{I}nforcement learning via war\textbf{M}-start d\textbf{E}noising discriminator. RIME consists of two main components: 1) a denoising discriminator designed to filter out corrupted preferences while accounting for training instability and distribution shifts, and 2) a warm start method to effectively initialize the reward model and enable a seamless transition from pre-training to online training. See Figure \ref{fig:overview} for the overview of RIME. The full procedure of RIME is detailed in Appendix \ref{app:algo}.

\subsection{Denoising Discriminator} \label{subsec:trust}
In the presence of noisy labels, it is well-motivated to distinguish between clean and corrupted samples for robust training. Existing research indicates that deep neural networks first learn generalizable patterns before overfitting to the noise in the data \citep{arpit2017closer,li2020gradient}. Therefore, prioritizing samples associated with smaller losses as clean ones is a well-founded approach to improve robustness. Inspired by this, a theoretical lower bound on the KL divergence between the predicted preference $P_\psi$ and the annotated preference $\tilde{y}$ for corrupted samples could be established to filter out large-loss corrupted samples.

\begin{theorem}[KL Divergence Lower Bound for Corrupted Samples]
\textit{Consider a preference dataset \(\{(\sigma^0_i,\sigma^1_i,\tilde{y}_i)\}_{i=1}^n\), where \(\tilde{y}_i\) is the annotated label for the segment pair \((\sigma^0_i,\sigma^1_i)\) with the ground truth label \(y_i\). Let \( x_i \) denote the tuple \( (\sigma^0_i, \sigma^1_i) \). Assume the cross-entropy loss \(\mathcal{L}^{\mathtt{CE}}\) for clean data (whose $\tilde{y}_i=y_i$) is bounded by \(\rho\). Then, the KL divergence between the predicted preference \(P_\psi(x)\) and the annotated label \(\tilde{y}(x)\) for a corrupted sample \(x\) is lower-bounded as:}
\begin{equation}
    \KL \left(\tilde{y}(x) \Vert P_\psi(x) \right) \geq -\ln \rho + \frac{\rho}{2} + \mathcal{O}(\rho^2)
\end{equation}
\vspace{-10pt}
\label{the:KL}
\end{theorem}
The proof of Theorem \ref{the:KL} is presented in Appendix \ref{app:proof_theorem_1}. Based on Theorem \ref{the:KL}, the lower bound on KL divergence threshold could be formulated  to filter out untrustworthy samples as \(\tau_{\text{base}} = -\ln \rho + \alpha \rho\) in practice, where \(\rho\) denotes the maximum cross-entropy loss on trustworthy samples observed during the last update, and \(\alpha\) is a tunable hyperparameter with a theoretically-determined value range in $(0, 0.5]$.

However, compared to deep learning, the shifting state distribution makes the robust training problem in RL more complicated. To add tolerance for clean samples in cases of distribution shift, we introduce an auxiliary term characterizing the uncertainty for filtration, defined as \(\tau_{\text{unc}} = \beta_t \cdot s_{\text{KL}}\), where \(\beta_t\) is a time-dependent parameter, and \(s_{\text{KL}}\) is the standard deviation of the KL divergence. Our intuition is that the inclusion of out-of-distribution data for training is likely to induce fluctuations in training loss. Therefore, the complete threshold equation is formulated as:
\begin{equation}
    \tau_{\text{lower}} = \tau_{\text{base}} + \tau_{\text{unc}} = -\ln \rho + \alpha \rho + \beta_t \cdot s_{\text{KL}}
    \label{eq:threshold}
\end{equation}
We utilize a linear decay schedule for \(\beta_t\) to initially allow greater tolerance for samples while becoming increasingly conservative over time, \ie $\beta_t=\max(\beta_{\min}, \beta_{\max}-kt)$. At each training step for the reward model, we apply the threshold in Equation (\ref{eq:threshold}) to identify trustworthy sample dataset $\mathcal{D}_t$, as described below: 
\begin{align}
\mathcal{D}_t=\{(\sigma^0,\sigma^1,\tilde{y})\,|\,\KL (\tilde{y} \Vert P_\psi(\sigma^0,\sigma^1))<\tau_{\text{lower}}\}
\label{eq:trust}
\end{align}
To ensure efficient usage of samples, we introduce a label-flipping method for the reintegration of untrustworthy samples. Specifically, we pre-define an upper bound \(\tau_{\text{upper}}\) and reverse the labels for samples exceeding this threshold:
\begin{align}
\resizebox{0.9\hsize}{!}{$
\mathcal{D}_f=\{(\sigma^0,\sigma^1,1-\tilde{y})\,|\,\KL (\tilde{y} \Vert P_\psi(\sigma^0,\sigma^1))>\tau_{\text{upper}}\}$}
\label{eq:flip}
\end{align}
Beyond improving sample utilization, the label-flipping method also bolsters the model's predictive confidence and reduces output entropy~\citep{grandvalet2004semi}. Following two filtering steps, the reward model is trained on the unified datasets $\mathcal{D}_t\cup\mathcal{D}_f$, using the loss function as follows: 
\begin{align}
    \mathcal{L}^{\mathtt{CE}} = &\expec_{(\sigma^0,\sigma^1,\tilde{y})\sim \mathcal{D}_t}\left[\mathcal{L}^{\mathtt{Reward}}(\sigma^0,\sigma^1,\tilde{y})\right] + \notag \\ &\expec_{(\sigma^0,\sigma^1,1-\tilde{y})\sim \mathcal{D}_f}\left[\mathcal{L}^{\mathtt{Reward}}(\sigma^0,\sigma^1,1-\tilde{y})\right]
    \label{eq:final loss}
\end{align}
Our denoising discriminator belongs to the category of sample selection methods for robust training. It stands out due to its use of a dynamically adjusted threshold enhanced by a term that accounts for instability and distributional shifts, thereby making it more suitable for the RL training process.

\subsection{Warm Start} \label{subsec:pretrain}
Sample selection methods usually suffer from accumulated errors due to incorrect selection, which highlights the need for good initialization.
Meanwhile, we observe a significant degradation in performance during the transition from pre-training to online training (see Figure \ref{fig:performance degradation}). This gap is clearly observable under noisy feedback settings and is fatal to robustness. It is exacerbated when following the most widely-adopted backbone, PEBBLE, which resets the Q-network and only retains the pre-trained policy after the pre-training phase. Because the Q-network is optimized with a biased reward model trained on noisy preferences to minimize the Bellman residual, this biased Q-function leads to a poor learning signal for the policy, erasing gains made during pre-training. 

Inspired by these observations, we propose to warm start the reward model to facilitate a smoother transition from pre-training to online training. Specifically, we pre-train the reward model to approximate intrinsic rewards during the pre-training phase. Because the output layer of the reward model typically uses the tanh activation function~\citep{lee2021pebble}, we firstly normalize the intrinsic reward to the range $(-1, 1)$ as follows:
\begin{align}
    r^{\text{int}}_{\text{norm}}(\state_t)=\text{clip}(\frac{r^{\text{int}}(\state_t)-\Bar{r}}{3\sigma_r},-1+\delta,1-\delta)
    \label{eq:norm intrinsic reward}
\end{align}
where $0<\delta\ll 1$. $\Bar{r}$ and $\sigma_r$ are the mean and standard deviation of the intrinsic rewards, respectively. Then the agent receives the reward $r^{\text{int}}_{\text{norm}}$ and stores each tuple $(\state_t, \action_t, r^{\text{int}}_{\text{norm}}, \state_{t+1})$ in a replay buffer, denoted as $\mathcal{D}_{\text{pretrain}}$. During the reward model update, we sample batches of $(\state_t, \action_t)$ along with all encountered states $\mathcal{S} = \{\state | \state \text{ in } \mathcal{D}_{\text{pretrain}}\}$ for nearest neighbor searches. The loss function for updating the reward model $\hat{r}_\psi$ is given by the mean squared error as:
\begin{align}    \mathcal{L}^{\mathtt{MSE}}=\expec_{(\state_t,\action_t)\sim\mathcal{D}_{\text{pretrain}}}\Big[\frac12\left( \hat{r}_\psi(\state_t,\action_t)-r^{\text{int}}_{\text{norm}}(\state_t)\right)^2\Big]
\label{eq:loss mse}
\end{align}
Thanks to warm start, both the Q-network and reward model are aligned with intrinsic rewards, allowing for the retention of all knowledge gained during pre-training (\ie policy, Q-network, and reward model) for subsequent online training. Moreover, the warm-started reward model contains more information than random initialization, enhancing the discriminator's ability initially.

\begin{figure}[H]
    \centering
    \includegraphics[width=0.23\textwidth]{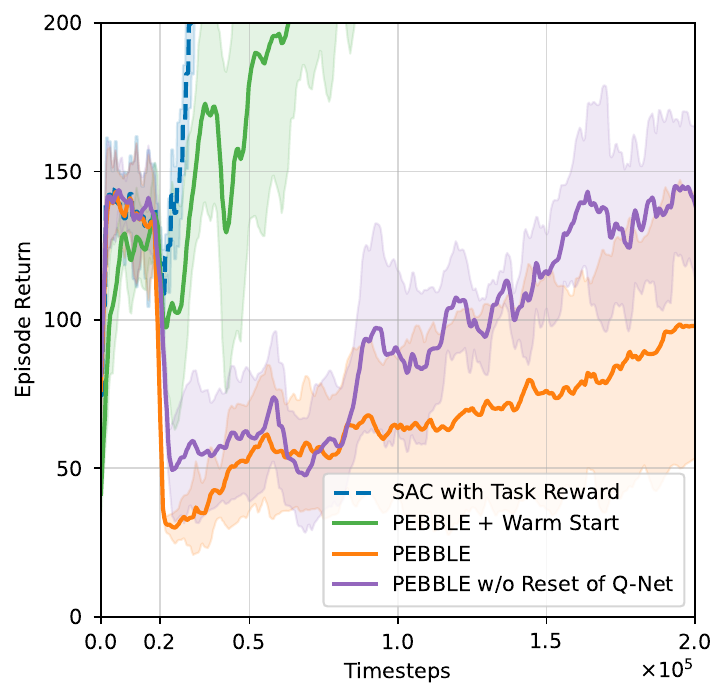}
    \includegraphics[width=0.23\textwidth]{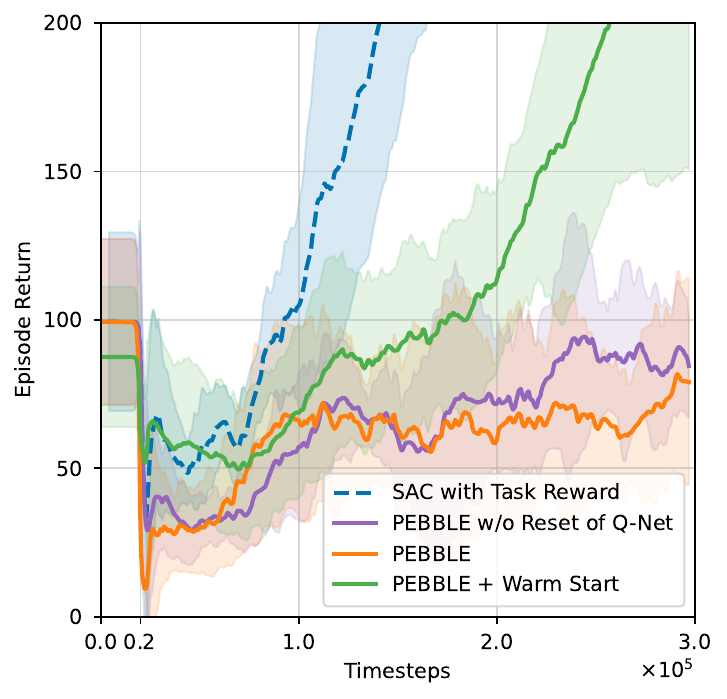}
    \caption{Performance degradation during transition on Walker-walk (left) and Quadruped-walk (right) with 30\% noisy preferences. We pre-train an agent using SAC for 20k steps. The warm start method shows a smaller transition gap and faster recovery.}
    \label{fig:performance degradation}
\end{figure}

\section{Experiments}

\begin{figure*}[t]
\centering
\captionsetup[subfloat]{captionskip=-8pt}
\includegraphics[width=0.9\linewidth]{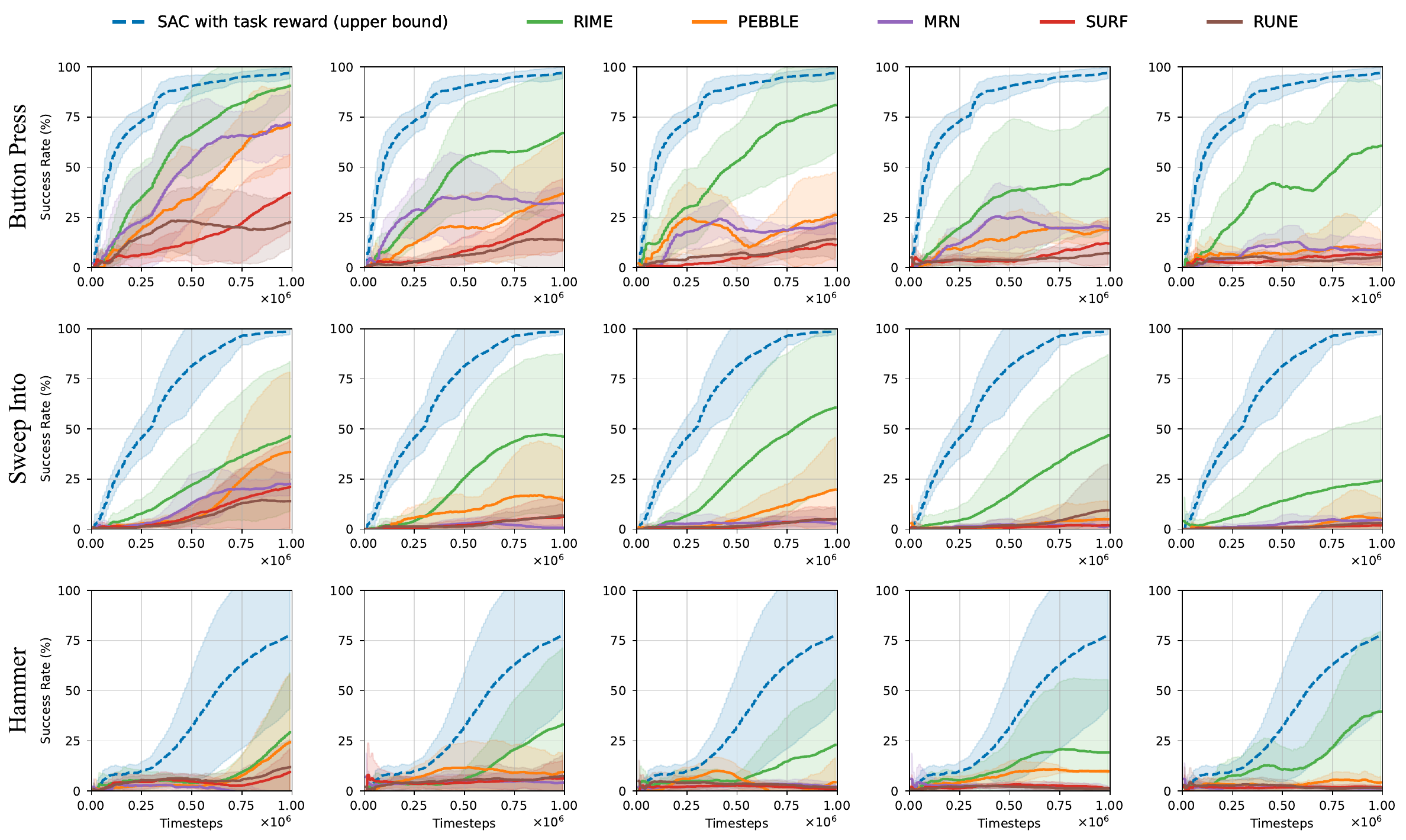}\\
\subfloat[$\epsilon=0.1$]{\hspace{0.28\linewidth}}
\subfloat[$\epsilon=0.15$]{\hspace{0.16\linewidth}}
\subfloat[$\epsilon=0.2$]{\hspace{0.16\linewidth}}
\subfloat[$\epsilon=0.25$]{\hspace{0.16\linewidth}}
\subfloat[$\epsilon=0.3$]{\hspace{0.23\linewidth}}
\caption{Learning curves for robotic manipulation tasks from Meta-world, where each row represents a specific task and each column corresponds to a different error rate $\epsilon$. 
SAC serves as a performance upper bound, using a ground-truth reward function unavailable in PbRL settings.
The corresponding number of feedback in total and per session are shown in Table \ref{table:hyperparameters_condition}. The solid line and shaded regions respectively denote the mean and standard deviation of the success rate, across ten runs.}
\label{figure:Metaworld learning curves}
\vspace{-5pt}
\end{figure*}

\begin{figure*}[t]
\centering
\captionsetup[subfloat]{captionskip=-8pt}
\includegraphics[width=0.9\linewidth]{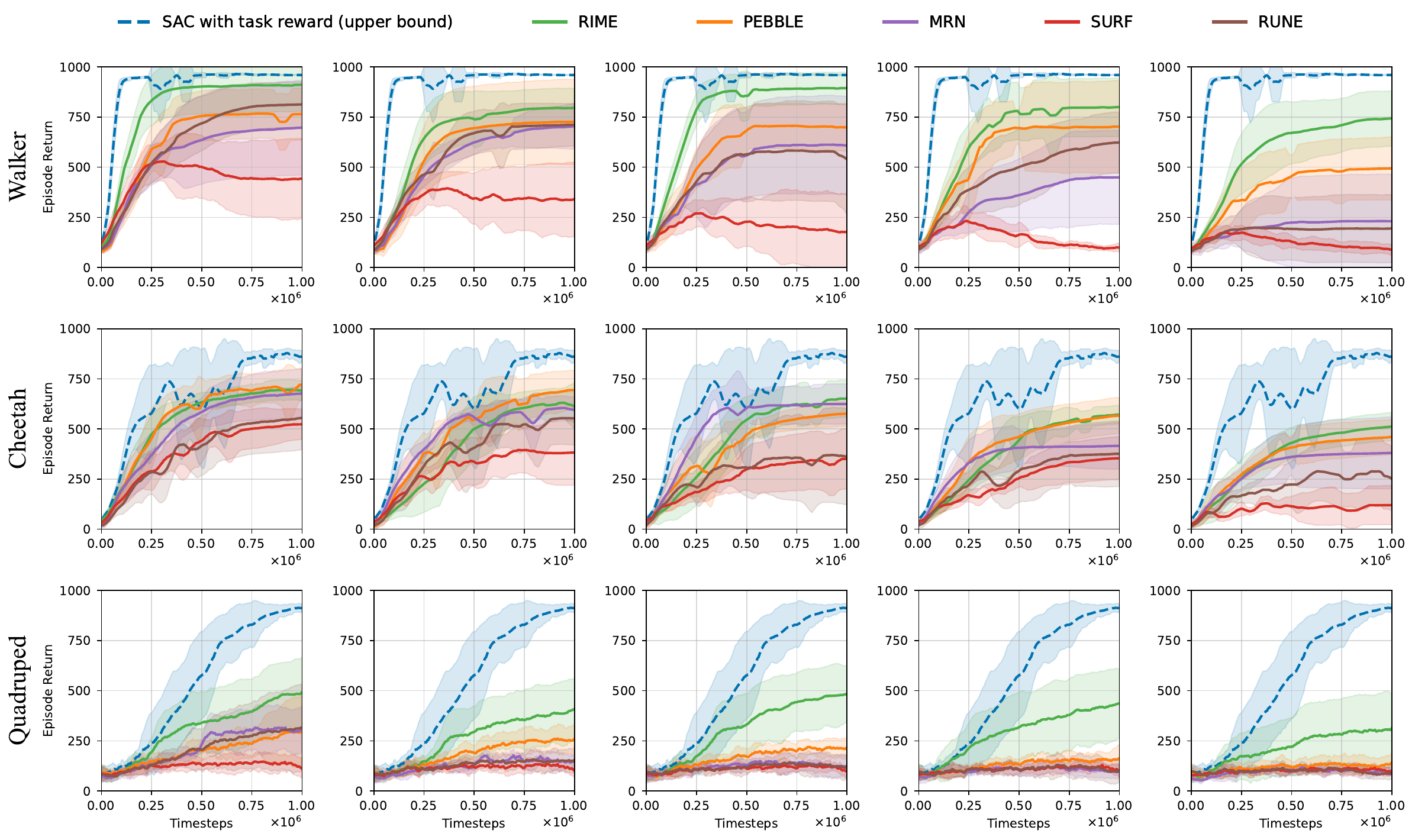}\\
\subfloat[$\epsilon=0.1$]{\hspace{0.28\linewidth}}
\subfloat[$\epsilon=0.15$]{\hspace{0.16\linewidth}}
\subfloat[$\epsilon=0.2$]{\hspace{0.16\linewidth}}
\subfloat[$\epsilon=0.25$]{\hspace{0.16\linewidth}}
\subfloat[$\epsilon=0.3$]{\hspace{0.23\linewidth}}
\caption{Learning curves on locomotion tasks from DMControl, where each row represents a specific task and each column corresponds to a different error rate $\epsilon$ setting. SAC serves as a performance upper bound, using a ground-truth reward function unavailable in PbRL settings.
The corresponding number of feedback in total and per session are shown in Table \ref{table:hyperparameters_condition}. The solid line and shaded regions respectively denote the mean and standard deviation of episode return, across ten runs.}
\label{figure:DMControl learning curves}
\end{figure*}

\subsection{Setups} \label{sec:setup}
We evaluate RIME on six complex tasks, including robotic manipulation tasks from Meta-world \citep{yu2020meta} and locomotion tasks from DMControl \citep{tassa2018deepmind, tassa2020dm_control}. The details of experimental tasks are shown in Appendix \ref{app:tasks}. Similar to prior works~\citep{lee2021b,lee2021pebble,park2021surf}, to ensure a systematic and fair evaluation, we consider a scripted teacher that provides preferences between two trajectory segments based on the sum of ground-truth reward values for each segment. To generate noisy preferences, we follow the procedure of the ``mistake" scripted teacher in \citet{lee2021b}, which flips correct preferences with a probability of $\epsilon$. We refer to $\epsilon$ as the error rate. We choose PEBBLE~\citep{lee2021pebble} as our backbone algorithm to implement RIME. In our experiments, we compare RIME against ground-truth reward-based SAC and four state-of-the-art PbRL algorithms: PEBBLE~\citep{lee2021pebble}, SURF~\citep{park2021surf}, RUNE~\citep{liang2021reward}, and MRN~\citep{liu2022meta}. Here, SAC is considered as an upper bound for performance, as it utilizes a ground-truth reward function not available in PbRL settings. We include SAC in our comparisons because it is the backbone RL algorithm of PEBBLE.

\textbf{Implementation Details.}
For the hyperparameters of RIME, we fix $\alpha=0.5$, $\beta_{\min}=1$ and $\beta_{\max}=3$ in the lower bound $\tau_\text{lower}$, and fix the upper bound $\tau_{\text{upper}}=3\ln(10)$ for all experiments. The decay rate $k$ in $\tau_{\text{upper}}$ is $1/30$  for DMControl tasks, and $1/300$ for Meta-world tasks, respectively. Other hyperparameters are kept the same as PEBBLE. For the sampling of queries, we use the disagreement sampling scheme for all PbRL algorithms, following the setting in \citet{christiano2017deep}. For the implementation of baselines, we use their corresponding publicly released repositories (see Table \ref{table:source_code} for source codes). The feedback amount in total and per query session in each environment with specified error rate are detailed in Table \ref{table:hyperparameters_condition}.

For each task, we run all algorithms independently ten times and report the average performance along with the standard deviation. Tasks from Meta-world are measured on success rate, while tasks from DMControl are measured on ground-truth episode return. More details on the algorithm implementation are provided in Appendix \ref{app:implemen_detail}. 

\begin{table}[H]
    \vspace{-10pt}
    \caption{Feedback amount in each environment with specified error rate. The ``value" column refers to the feedback amount in total / per session.}
    \centering
    \setlength{\tabcolsep}{1.0mm}
    \renewcommand{\arraystretch}{1.05}
    \resizebox{\columnwidth}{!}{
    \begin{tabular}{lll|lll}
\specialrule{.1em}{.05em}{.05em}
\textbf{Environment} & \textbf{Error rate} & \textbf{Value} & \textbf{Environment} & \textbf{Error rate} & \textbf{Value} \\
\midrule
Walker & $\epsilon<0.2$ & $500/50$ & Button Press & $\epsilon<0.2$ & $10000/50$ \\
Walker & $\epsilon\geq0.2$ & $1000/100$ & Button Press & $\epsilon\geq0.2$ & $20000/100$ \\
Cheetah & $\epsilon<0.2$ & $500/50$ & Sweep Into & $\epsilon<0.2$ & $10000/50$ \\
Cheetah & $\epsilon\geq0.2$ & $1000/100$ & Sweep Into & $\epsilon\geq0.2$ & $20000/100$ \\
Quadruped & $\epsilon<0.2$ & $2000/200$ & Hammer & $\epsilon<0.2$ & $20000/100$  \\
Quadruped & $\epsilon\geq0.2$ & $4000/400$ & Hammer & $\epsilon=0.2,0.25$ & $40000/200$ \\
 &  & & Hammer & $\epsilon=0.3$ & $80000/400$ \\
\specialrule{.1em}{.05em}{.05em}
    \end{tabular}}
    \label{table:hyperparameters_condition}
    \vspace{-10pt}
\end{table}

\subsection{Results}

\begin{table*}[tb]
    \small
    \begin{center}
    \caption{Results on tasks from Meta-world and DMControl with ``mistake" teacher. The result shows the mean and standard deviation of metric (\ie episode return for DMControl tasks and success rate for Meta-world tasks) across all five error rates within 10 runs.}
    \setlength{\tabcolsep}{2.5mm}
    \renewcommand{\arraystretch}{1.1}
    \begin{tabular}{l|ccc|ccc}
    \toprule
\multirow{2}{*}{Algorithm} & \multicolumn{3}{c|}{DMControl} & \multicolumn{3}{c}{Meta-world}     \\  \cline{2-7}
~ & Walker  & Cheetah & Quadruped & Button Press & Sweep Into & Hammer 
\\ \hline
PEBBLE & 692.05 \stdv{192.67} & \textbf{604.77} \stdv{126.63} & 208.66 \stdv{106.81} & 50.07 \stdv{29.53} & 19.09 \stdv{29.82} & 26.22 \stdv{32.08} \\
SURF & 211.28 \stdv{195.25} & 341.43 \stdv{178.10} & 125.51 \stdv{\textcolor{white}{0}40.15} & 42.60 \stdv{27.55} & 16.04 \stdv{22.86} & 11.43 \stdv{22.76} \\
RUNE & 584.06 \stdv{271.84} & 424.17 \stdv{205.16} & 152.66 \stdv{131.43} & 27.04 \stdv{18.89} & 15.02 \stdv{19.18} & 12.14 \stdv{19.30} \\
MRN & 537.40 \stdv{281.36} & 538.74 \stdv{169.63} & 139.65 \stdv{\textcolor{white}{0}88.24} & 43.48 \stdv{30.58} & 14.74 \stdv{22.89} & \textcolor{white}{0}6.35 \stdv{\textcolor{white}{0}9.55} \\
\rowcolor{mygray} RIME & \textbf{837.79} \stdv{133.49} & \textbf{602.18} \stdv{\textcolor{mygray}{0}96.10} & \textbf{415.52} \stdv{180.74} & \textbf{85.70} \stdv{22.92} & \textbf{51.96} \stdv{42.90} & \textbf{42.28} \stdv{42.31}
\\ \bottomrule
\end{tabular}
\label{table:main experiment}
\end{center}
\vspace{-5pt}
\end{table*}

For robotic manipulation tasks, we consider three tasks from Meta-world: Button-press, Sweep-into, and Hammer. For locomotion tasks, we choose three environments from DMControl: Walker-walk, Cheetah-run, and Quadruped-walk. Figure \ref{figure:Metaworld learning curves} and Figure \ref{figure:DMControl learning curves} show the learning curves of RIME and baselines on Meta-world and DMControl tasks with five error rates, respectively. Table \ref{table:main experiment} shows the mean and standard deviation of metrics across the five error rates. 

Since some preferences are corrupted, we observe that there is a gap between all PbRL methods and the best performance (\ie SAC with task reward), but RIME exceeds the PbRL baselines by a large margin in almost all environments. Especially, RIME remains effective in cases where all baselines struggle, such as Button-press with $\epsilon=0.2$, Hammer with $\epsilon=0.3$, and Walker with $\epsilon=0.3$, etc. These results demonstrate that RIME significantly improves robustness against noisy preferences. We also observe that although some feedback-efficient baselines based on PEBBLE perform comparable to or even exceed PEBBLE in low-level noise, they become ineffective as the error rate rises. Additionally, Table \ref{table:main experiment} shows that PEBBLE is a robust algorithm second only to RIME. These results reveal that the pursuit of feedback efficiency leads to over-reliance on feedback quality. 

\subsection{Ablation Study} \label{subsec:ablation study}
\begin{table*}[htb]
    \small
    \begin{center}
    \caption{Results on tasks from Meta-world and DMControl with 4 other types of (noisy) teacher. The result shows the mean and standard deviation of metric (\ie episode return for DMControl tasks and success rate for Meta-world tasks) averaged over 5 runs.}
    \setlength{\tabcolsep}{2.0mm}
    \renewcommand{\arraystretch}{1.1}
\begin{tabular}{l|c|l|cccc|c}
\toprule
Domain & Environment & Algorithm & Oracle & Equal & Skip & Myopic & Average \\ \hline
\multirow{6}{*}{DMControl} & \multirow{3}{*}{Walker} & PEBBLE & 877.44 \stdv{44.06}  & 930.90 \stdv{17.77}  & 904.31 \stdv{26.59}  & 762.53 \stdv{165.98} & 868.80  \\
                             & & MRN       & 913.66 \stdv{51.84}  & 942.80 \stdv{14.14}  & 919.61 \stdv{48.87}  & 882.34 \stdv{\textcolor{white}{0}19.68}  & 914.60  \\
                             
                             & & \cellcolor{mygray}RIME      & \cellcolor{mygray}958.87 \stdv{\textcolor{mygray}{0}3.08}   & \cellcolor{mygray}954.89 \stdv{\textcolor{mygray}{0}1.43}  & \cellcolor{mygray}950.83 \stdv{16.44}  & \cellcolor{mygray}952.16 \stdv{\textcolor{mygray}{00}1.80}    & \cellcolor{mygray}\textbf{954.19} \\ 
                             \cline{2-8} 
                             
~ & \multirow{3}{*}{Quadruped}    & PEBBLE    & 620.35 \stdv{193.74} & 743.04 \stdv{107.30} & 776.01 \stdv{\textcolor{white}{0}65.86}  & 622.78 \stdv{200.04} & 690.55  \\
                             & & MRN       & 682.98 \stdv{182.25} & 666.56 \stdv{298.40} & 653.28 \stdv{150.78} & 525.91 \stdv{233.77} & 633.18  \\
                             
                             & & \cellcolor{mygray}RIME      & \cellcolor{mygray}678.36 \stdv{\textcolor{mygray}{0}33.02}  & \cellcolor{mygray}784.05 \stdv{\textcolor{mygray}{0}56.96} & \cellcolor{mygray}755.58 \stdv{116.24} & \cellcolor{mygray}688.44 \stdv{130.59} & \cellcolor{mygray}\textbf{726.61}  \\ 
                             \hline
                             
\multirow{6}{*}{Meta-world} & \multirow{3}{*}{Button-Press} & PEBBLE    & 100.0 \stdv{0.0}       & 100.0 \stdv{0.0}      & 100.0 \stdv{0.0}       & \textcolor{white}{0}99.8 \stdv{0.4}      & 99.95   \\
                             & & MRN       & 100.0 \stdv{0.0}       & 100.0 \stdv{0.0}      & \textcolor{white}{0}99.6 \stdv{0.5}     & 100.0 \stdv{0.0}       & 99.90   \\
                             & & \cellcolor{mygray}RIME      & \cellcolor{mygray}100.0 \stdv{0.0}       & \cellcolor{mygray}100.0 \stdv{0.0}      & \cellcolor{mygray}100.0 \stdv{0.0}       & \cellcolor{mygray}100.0 \stdv{0.0}       & \cellcolor{mygray}\textbf{100.00}  \\
                             \cline{2-8}
                             
~ & \multirow{3}{*}{Hammer}       & PEBBLE    & 37.46 \stdv{44.95}   & 53.20 \stdv{34.20}  & 55.40 \stdv{33.97}    & 48.40 \stdv{40.39}    & 48.62   \\
                             & & MRN       & 67.20 \stdv{39.92}    & 44.13 \stdv{34.86}  & 52.20 \stdv{23.22}   & 41.60 \stdv{33.97}    & 51.28   \\
                             & & \cellcolor{mygray}RIME      & \cellcolor{mygray}56.00 \stdv{27.28}   & \cellcolor{mygray}53.80 \stdv{36.17}   & \cellcolor{mygray}54.80 \stdv{34.25}    & \cellcolor{mygray}70.60 \stdv{38.95}    & \cellcolor{mygray}\textbf{58.80} \\
                             \bottomrule
\end{tabular}
\label{table:more noisy teachers}
\end{center}
\vspace{-5pt}
\end{table*}

\textbf{Performance with more types of (noisy) teachers.}
To investigate whether our method can generalize to more situations, we evaluate RIME, PEBBLE, and MRN with the other four types of teachers proposed by \citet{lee2021b}: Oracle, Skip, Equal, and Myopic. ``Oracle" teacher provides ground-truth preferences. ``Skip" teacher will skip the query if the cumulative rewards of segments are small. ``Equal" teacher will give equal preference $\tilde{y}=(0.5,0.5)$ if the difference between the cumulative rewards of two segments is small. ``Myopic" teacher focuses more on the behavior at the end of segments. More details of these four teachers are shown in Appendix \ref{app:scripted teachers}. We report mean and standard deviation across five runs in Table \ref{table:more noisy teachers}. We found that RIME not only performs the best when teachers can provide ambiguous or wrong labels (Equal and Myopic), but it is also comparable with baselines on correct labels (Oracle and Skip). Based on the superior performance of RIME with multiple teachers, it has better chances of performing well with real teachers as well~\citep{lee2021b}.

\textbf{Comparison with other robust training methods.} Since the reward learning in PbRL is posed as a classification problem, the robust training methods in Machine Learning could be migrated to compare with RIME. Consider a sample selection method: adaptive denoising training (ADT) \citep{wang2021denoising}, two robust loss functions: Mean Absolute Error (MAE)~\citep{ghosh2017robust} and t-CE~\citep{feng2021can}, and a robust regularization method: label smoothing (LS)~\citep{wei2021smooth}, as our baselines. ADT drops a-$\tau(t)$ proportion of samples with the largest cross-entropy loss at each training iteration, where $\tau(t)=\min(\gamma t,\tau_{\max})$. We set $\tau_{\max}=0.3$, $\gamma=0.003$, and $0.0003$ for tasks from DMControl and Meta-world, respectively. MAE loss if formulated as $\mathcal{L}^{\mathtt{MAE}}(\psi) = \expec [|\tilde{y} - P_\psi |]$, while t-CE loss is formulated as $\mathcal{L}^{\mathtt{t}\text{-}\mathtt{CE}}(\psi) = \expec [\sum_{i=1}^t\frac{(1-\tilde{y}^\top\cdot P_\psi)^i}{i}]$. Label smoothing method replace $\tilde{y}$ in Equation (\ref{eq:CE loss}) with $(1-r)\cdot\tilde{y}+\frac{r}{2}\cdot[1,1]^\top$. We adopt $t=4$ for t-CE loss and $r=0.1$ for label smoothing respectively. All baselines are implemented based on PEBBLE. Table \ref{table:robust training methods} shows the result on four tasks with ``Mistake" teacher and error rate as $\epsilon=0.3$. Additional experiments with fixed lower bound $\tau_{\text{lower}}$ are provided in Appendix \ref{app:addition_result}. We observe that label smoothing almost fails to handle corrupted labels in our experiments. Sample selection methods (RIME and ADT) work better compared to other types of methods, and RIME still outperforms baselines. The reason is that the dynamic threshold with tolerance for out-of-distribution data is particularly well-suited to the RL training process.

\begin{table}[ht]
    \vspace{-10pt}
    \small
    \begin{center}
    \caption{Results of different robust training methods with ``mistake" teacher and error rate as $\epsilon=0.3$.}
    \setlength{\tabcolsep}{1.0mm}
    \renewcommand{\arraystretch}{1.1}
\begin{tabular}{l|cc|cc}
\specialrule{.1em}{.05em}{.05em}
\multirow{2}{*}{Algorithm} & \multicolumn{2}{c|}{DMControl} & \multicolumn{2}{c}{Meta-world}     \\  \cline{2-5}
~ & Walker & Quadruped & Button Press & Hammer 
\\ \hline
\color{gray}{PEBBLE}    & \color{gray}{431 \stdv{157}} & \color{gray}{125 \stdv{\textcolor{white}{0}38}}  & \color{gray}{22.0 \stdv{13.8}}    & \color{gray}{\textcolor{white}{0}8.6 \stdv{\textcolor{white}{0}4.8}}   \\
\quad + ADT       & 572 \stdv{247} & \textbf{295} \stdv{194} & \textbf{74.1} \stdv{20.9} & 37.6 \stdv{26.1} \\
\quad + MAE       & 453 \stdv{295}  & 246 \stdv{\textcolor{white}{0}22}  & 71.2 \stdv{31.0}  & 17.8 \stdv{26.9} \\
\quad + t-CE      & 548 \stdv{240} & 234 \stdv{\textcolor{white}{0}47}   & 36.0 \stdv{35.2}     & 20.2 \stdv{31.4} \\
\quad + LS        & 425 \stdv{172} & 117 \stdv{\textcolor{white}{0}32}  & 27.8 \stdv{21.0}  & \textcolor{white}{0}4.2 \stdv{\textcolor{white}{0}2.3}   \\ \hline
\rowcolor{mygray} RIME      & \textbf{741} \stdv{139}  & \textbf{301} \stdv{184} & \textbf{80.0} \stdv{27.7}  & \textbf{58.5} \stdv{42.0} \\
\specialrule{.1em}{.05em}{.05em}
\end{tabular}
\label{table:robust training methods}
\end{center}
\vspace{-10pt}
\end{table}

\textbf{Performance with real non-expert human teachers.}
The ultimate goal of improving robustness in PbRL is to better align with human users. To investigate how RIME performs with non-expert humans, we conduct experiments on Hopper utilizing non-expert human instructors, following the approach of \citet{christiano2017deep,lee2021pebble}. Specifically, we selected five students with no prior knowledge on robotics from unrelated majors to provide annotations in an online setting. These students were instructed solely on the objective: to train an agent to perform backflips and received no additional information or guidance on the task. Their annotations were later used to train the algorithms RIME and PEBBLE. The feedback amount in total and per session are 100 and 10 respectively. For further details on the annotation protocol, refer to Appendix \ref{app:annotation protocol}. 

We employ a hand-crafted reward function designed by experts \citep{christiano2017deep} as the ground-truth scripted teacher. We notice that compared to ground-truth preferences, the error rate of our non-expert annotations reached nearly 40\%. The learning curves are shown in Fig. \ref{fig:human teacher_a}. We find that RIME significantly outperforms PEBBLE when learning from actual non-expert human teachers and successfully performs consecutive backflips using 100 non-expert feedback, as shown in Figure \ref{fig:human teacher_b}. 

\begin{figure}[ht]
    \vspace{-5pt}
  \centering
  \begin{minipage}{.4\columnwidth}
    \centering
    \includegraphics[width=\columnwidth]{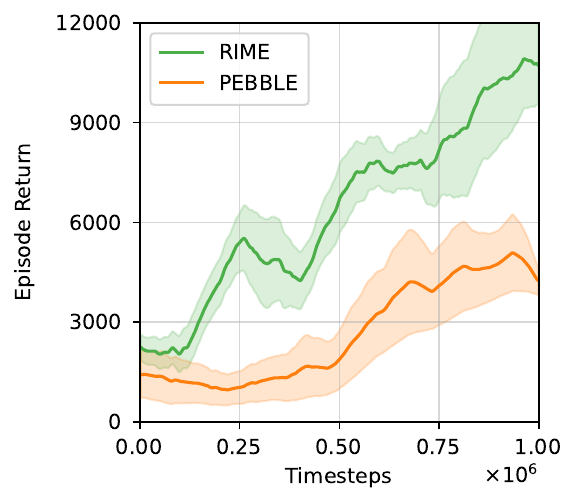}
  \end{minipage}
  \begin{minipage}{.56\columnwidth}
    \centering
    \includegraphics[width=\columnwidth]{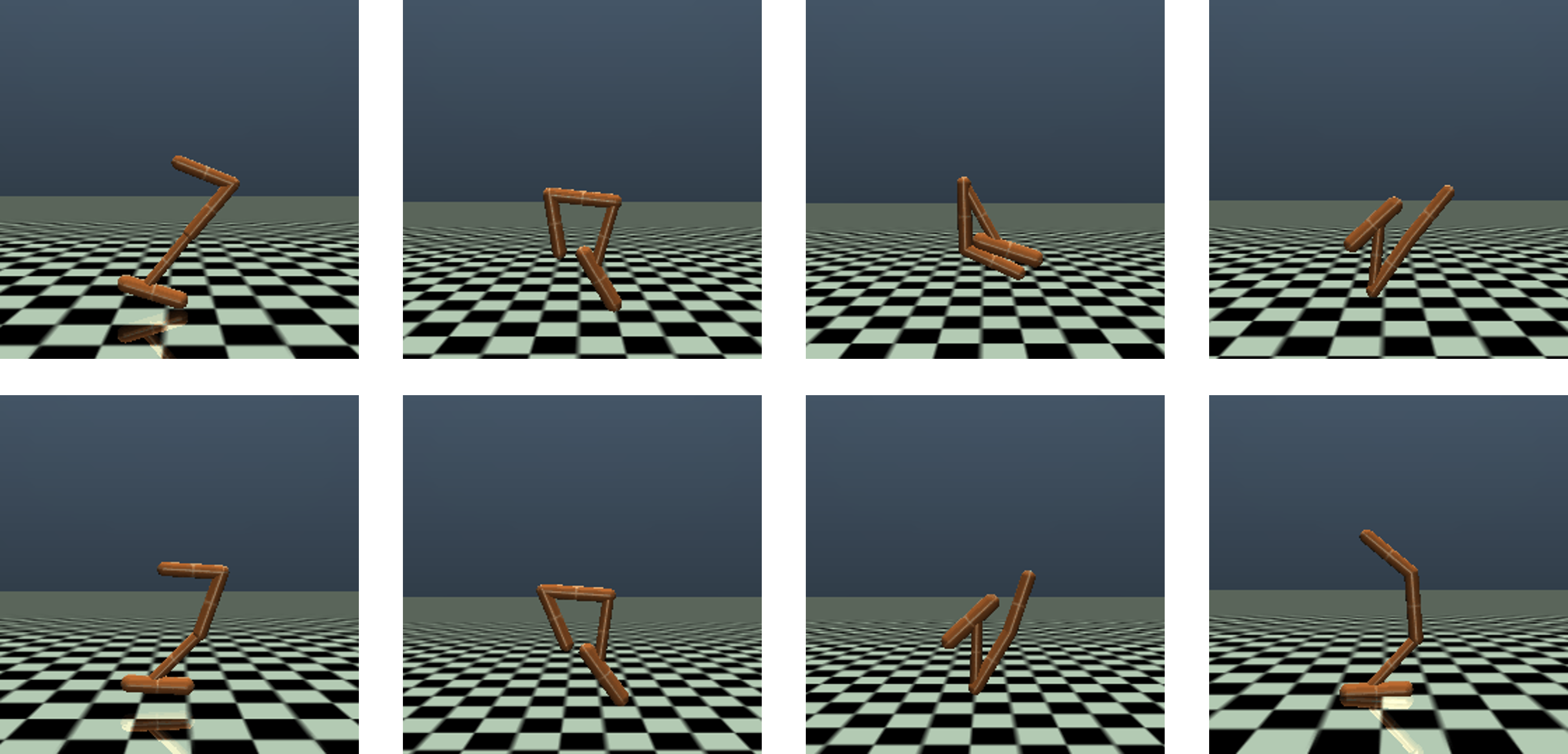}
  \end{minipage} \\
  \vspace{-10pt}
  \subfloat[\centering Learning curves of RIME and PEBBLE on Hopper]{\hspace{.43\columnwidth} \label{fig:human teacher_a}} 
  \subfloat[\centering Frames of consecutive backflips from agents trained by RIME.]{\hspace{.56\columnwidth} \label{fig:human teacher_b}}
  \caption{Ablation study on real non-expert human teachers}
\end{figure}

\textbf{Trade-off between sample efficiency and robustness.} To quantify the trade-off between sample efficiency and robustness, we conduct experiments with RIME, where we either held the feedback volume constant and increase the error rate, or maintain the error rate while increasing the feedback volume, as shown in Tables \ref{table:trade off error rate} and \ref{table:trade off amount}, respectively. Table \ref{table:trade off error rate} indicates a progressive decline in performance with rising error rates, with a notable deterioration at an error rate of 0.3 across all four environments. Table \ref{table:trade off amount} demonstrates that increasing the feedback volume improves performance, particularly at an error rate of 0.3, where doubling the feedback approximately doubles the performance gains. The same analysis repeated for PEBBLE is shown in Appendix \ref{app:addition_result}.

\begin{table}[h]
    \vspace{-10pt}
    \small
    \setlength{\belowcaptionskip}{1.0pt}
    \begin{center}
    \caption{Results of RIME as the error rate increases with constant amount of feedback, across 5 runs.}
    \setlength{\tabcolsep}{1mm}{
    \begin{tabular}{l|c|ccc}
    \specialrule{.1em}{.05em}{.05em}
        \multirow{2}{*}{Environment} & Feedback & \multicolumn{3}{c}{Error rate} \\ \cline{3-5}
         & volume & 0.1 & 0.2 & 0.3 \\
        \hline
        Walker & 500 & \textbf{909} \stdv{132} & 806 \stdv{162} & 493 \stdv{91} \\
        Quadruped & 2000 & \textbf{484} \stdv{166} & 400 \stdv{166} & 117 \stdv{24} \\
        Button-press & 10000 & \textbf{99.9} \stdv{0.3} & 92.2 \stdv{15.6} & 38.8 \stdv{33.8} \\
        Hammer & 20000 & \textbf{46.5} \stdv{43.9} & 41.2 \stdv{48.0} & 1.8 \stdv{3.1} \\
        \specialrule{.1em}{.05em}{.05em}
    \end{tabular}
    \label{table:trade off error rate}}
    \end{center}
    \vspace{-5pt}
\end{table}

\begin{table}[h]
    \vspace{-10pt}
    \small
    \setlength{\belowcaptionskip}{1.0pt}
    \begin{center}
    \caption{Results of RIME as the feedback volume increases with constant error rate, across 5 runs. $N$ refers to the minimal feedback volume for each environment shown in Table \ref{table:hyperparameters_condition}.}
    \setlength{\tabcolsep}{1mm}{
    \begin{tabular}{l|c|c|cc}
    \specialrule{.1em}{.05em}{.05em}
        \multirow{2}{*}{Domain} & \multirow{2}{*}{Environment} & \multirow{2}{*}{Error rate} & \multicolumn{2}{c}{Feedback volume} \\ \cline{4-5}
         & & & $N$ & 2$N$ \\
        \hline
        \multirow{4}{*}{DMControl} & \multirow{2}{*}{Walker} & 0.2 & 806 \stdv{162} & \textbf{894} \stdv{80} \\
         & & 0.3 & 493 \stdv{91} & \textbf{741} \stdv{139} \\ \cline{2-5}
         & \multirow{2}{*}{Quadruped} & 0.2 & 400 \stdv{166} & \textbf{477} \stdv{152} \\
         & & 0.3 & 117 \stdv{24} & \textbf{301} \stdv{184} \\ \hline
        \multirow{2}{*}{Meta-world} & \multirow{2}{*}{Button-press} & 0.2 & 92.2 \stdv{15.6} & \textbf{93.4} \stdv{13.6} \\
         & & 0.3 & 38.8 \stdv{33.8} & \textbf{80.0} \stdv{27.7} \\
        \specialrule{.1em}{.05em}{.05em}
    \end{tabular}
    \label{table:trade off amount}}
    \end{center}
    \vspace{-15pt}
\end{table}

\textbf{Component analysis.} 
We perform an ablation study to individually evaluate each technique in RIME: warm start (WS), lower bound $\tau_\text{lower}$, and upper bound $\tau_\text{upper}$ of KL divergence. We present results in Table \ref{table:component analysis} in which we compare the performance of removing each component from RIME. We observe that warm start is crucial for robustness when the number of feedback is quite limited (\ie on Walker-walk). This is because the limited samples restrict the capability of the reward model, leading to more rounds of queries to cross the transition gap. Moreover, identifying whether the sample is corrupted or not is challenging for the discriminator initially due to the limited training data, underscoring the need for well-initialized reward models.

The lower bound $\tau_\text{lower}$ for filtering trustworthy samples is important in high error rate ($\epsilon=0.3$ on Walker-walk and Button-press) and adequate feedback (on Button-press) situations. The upper bound $\tau_\text{upper}$ for flipping labels always brings some improvements in our ablation experiments. The full algorithm outperforms every other combination in most tasks. Additionally, the results show that although the contribution of warm start and denoising discriminator vary in different environments with low-level noise, they are both effective and their combination proves essential for the overall success of RIME in environments with high-level noise.

\begin{table}[H]
    \vspace{-5pt}
    \small
    \setlength{\belowcaptionskip}{1.0pt}
    \begin{center}
    \caption{Ablation study of components in RIME on Walker and Button-press with different error rates, across 5 runs.}
    \setlength{\tabcolsep}{1mm}{
    \begin{tabular}{ccc|cc|cc}
    \specialrule{.1em}{.05em}{.05em}
        \multicolumn{3}{c|}{Component}  & \multicolumn{2}{c|}{Walker} & \multicolumn{2}{c}{Button Press} \\
        \hline
        WS & $\tau_\text{lower}$ & $\tau_\text{upper}$ & $\epsilon=0.1$ & $\epsilon=0.3$ & $\epsilon=0.1$ & $\epsilon=0.3$ \\
        \hline
         \ding{55} & \ding{55} & \ding{55} & 749 \stdv{123} & 431 \stdv{157} & 93.1 \stdv{10.6} & 22.0 \stdv{13.8} \\
         \ding{51} & \ding{55} & \ding{55} & 821 \stdv{\textcolor{white}{0}93} & 483 \stdv{144} & 92.7 \stdv{11.8} & 25.8 \stdv{16.2} \\
         \ding{55} & \ding{51} & \ding{51} & 688 \stdv{148} & 457 \stdv{190} & 97.2 \stdv{\textcolor{white}{0}4.6} & 64.7 \stdv{26.5} \\
         \ding{51} & \ding{55} & \ding{51} & 886 \stdv{\textcolor{white}{0}70} & 492 \stdv{188} & 89.8 \stdv{11.5} & 35.1 \stdv{24.1} \\
         \ding{51} & \ding{51} & \ding{55} & 842 \stdv{107} & 693 \stdv{167} & 96.9 \stdv{\textcolor{white}{0}4.0} & 51.4 \stdv{30.0} \\
        \rowcolor{mygray} \ding{51} & \ding{51} & \ding{51} & \textbf{909} \stdv{132} & \textbf{741} \stdv{139} & \textbf{99.9} \stdv{\textcolor{mygray}{0}0.3} & \textbf{80.0} \stdv{27.7} \\
        \specialrule{.1em}{.05em}{.05em}
    \end{tabular}
    \label{table:component analysis}}
    \end{center}
\end{table}

\section{Conclusion}
In this paper, we present RIME, a robust algorithm for preference-based reinforcement learning (PbRL) designed for effective reward learning from noisy preferences. Unlike previous research which primarily aims to enhance feedback efficiency, RIME focuses on improving robustness by employing a sample selection-based discriminator to dynamically denoise preferences. To reduce accumulated error due to incorrect selection, we utilize a warm-start method for the reward model, enhancing the initial capability of the denoising discriminator. The warm-start approach also facilitates a seamless transition from pre-training to online training. Our experiments show that RIME substantially boosts the robustness of the state-of-the-art PbRL method across a wide range of complex robotic manipulation and locomotion tasks. Ablation studies further demonstrate that the warm-start approach is crucial for both robustness and feedback efficiency. We believe that RIME has the potential to broaden the applicability of PbRL by leveraging preferences from non-expert users or crowd-sourcing platforms.

\textbf{Limitations.} 
The intrinsic challenge of PbRL with noisy preferences is the trade-off between sample-efficiency and robustness. As shown in Table \ref{table:hyperparameters_condition}, \ref{table:trade off error rate}, and \ref{table:trade off amount}, RIME still needs to increase the amount of feedback in total to perform reasonably as the error rate increases. Therefore, exploring how to achieve better trade-offs is an interesting future direction. Another limitation is that RIME introduces several hyperparameters in the denoising discriminator, resulting in extra tuning efforts for optimal performance. Moreover, the divergence between the noise in actual human preferences and the noise models proposed by BPref deserves further investigation. While attaining high performance across multiple simulated teachers suggests a likelihood of good performance under real human teachers, it remains imperative to investigate simulated noise models that more closely align with the characteristics of real-world human preference noise.

\section*{Impact Statement}
Compared to natural language processing, control tasks typically demand higher-quality human feedback~\citep{kim2022preference}. Our work reduces the difficulty of annotating human preferences for control tasks, allowing for the presence of noise in preferences and thereby alleviating the requirement of domain knowledge for annotators. This enables human preferences for control tasks to be sourced from crowd-sourcing platforms or ordinary users, potentially expanding the application scope of PbRL.

\section*{Acknowledgments}
We are grateful to the anonymous reviewers and editors for their insightful suggestions. This work was partially supported by the National Science and Technology Major Project (2022ZD0117102), the National Natural Science Foundation of China under Grants 62271485, U1909204 and 62303462, the Provincial Key Research and Development Program of Zhejiang (project number: 2022C01129), Ningbo International Science and Technology Cooperation Project (2023H020), and Beijing Natural Science Foundation under Grant L233005. Lastly, special thanks go to Tongtian Yue, Ruixi Qiao, and Yueyang Yang for their assistance in data annotation, engaging discussions, and meticulous review of draft versions.

\bibliography{example_paper}
\bibliographystyle{icml2024}

\newpage
\appendix
\onecolumn

\section{RIME Algorithm Details} \label{app:algo}
In this section, we provide the full procedure for RIME based on the backbone PbRL algorithm, PEBBLE~\citep{lee2021pebble}, in Algorithm~\ref{alg:training}.
\begin{algorithm}
\caption{RIME} \label{alg:training}
\begin{algorithmic}[1]
    \State Initialize policy $\policy_\phi$, Q-network $Q_\theta$ and reward model $\hat{r}_\psi$
    \State Initialize replay buffer $\mathcal{B}\leftarrow\emptyset$
    \For{each pre-training step $t$} \Comment{\textsc{Unsupervised Pre-training}}
        \State Collect $\state_{t+1}$ by taking $\action_t\sim\policy_\phi(\action_t|\state_t)$
        \State Compute normalized intrinsic reward $r^\text{int}_{\text{norm},t}\leftarrow r^\text{int}_\text{norm}(\state_t)$ as in Equation (\ref{eq:norm intrinsic reward})
        \State Store transitions $\mathcal{B}\leftarrow\mathcal{B}\cup\left\{(\state_t,\action_t,\state_{t+1},r^\text{int}_{\text{norm},t})\right\}$
        \For{each gradient step}
            \State Sample minibatch $\left\{(\state_j,\action_j,\state_{j+1},r^\text{int}_{\text{norm},j})\right\}^B_{j=1}\sim\mathcal{B}$
            \State Optimize policy and Q-network with respect to $\phi$ and $\theta$ using SAC
            \State Update reward model $\hat{r}_\psi$ according to Equation (\ref{eq:loss mse}) \Comment{{{\textsc{Warm start}}}}
        \EndFor
    \EndFor
    
    \State Initialize the maximum KL divergence value $\rho=\infty$ 
    \State Initialize a dataset of noisy preferences $\mathcal{D}_\text{noisy}\leftarrow\emptyset$
    \For{each training step $t$} \Comment{{{\textsc{Online training}}}}
        \If{step to query preferences} \Comment{{{\textsc{Robust reward learning}}}}
        \State Generate queries from replay buffer $\{(\sigma^0_i, \sigma^1_i)\}_{i=1}^{N_\text{query}} \sim\mathcal B$ and corresponding human feedback $\{\tilde{y}_i\}_{i=1}^{N_\text{query}}$
        \State Store preferences $\mathcal{D}_\text{noisy}\leftarrow\mathcal{D}_\text{noisy}\cup\{(\sigma^0_i,\sigma^1_i,\tilde{y}_i)\}_{i=1}^{N_\text{query}}$
        \State Compute lower bound $\tau_\text{lower}$ according to Equation (\ref{eq:threshold}) \State Filter trustworthy samples $\mathcal{D}_t$ using lower bound $\tau_\text{lower}$ as in Equation (\ref{eq:trust})
        \State Flip labels using upper bound $\tau_\text{upper}$ to obtain dataset $\mathcal{D}_f$ as in Equation (\ref{eq:flip})
        \State Update reward model $\hat{r}_\psi$ with samples from $\mathcal{D}_t\cup\mathcal{D}_f$ according to Equation (\ref{eq:final loss})
        \State Relabel entire replay buffer $\mathcal{B}$ using $\hat{r}_\psi$
        \State Update parameter $\rho$ with the maximum KL divergence between predicted and annotated labels in dataset $\mathcal{D}_t\cup\mathcal{D}_f$
        \EndIf
        \State Collect $\state_{t+1}$ by taking $\action_t \sim \policy_\phi(\action_t | \state_t)$
        \State Store transitions $\mathcal{B} \leftarrow \mathcal{B}\cup \{(\state_t,\action_t,\state_{t+1},\rewmodel(\state_t, \action_t))\}$
        \For{each gradient step}
            \State Sample minibatch from replay buffer $\{(\state_{j}, \action_{j}, \state_{j+1}, \hat{r}_\psi(\state_j,\action_j)\}_{j=1}^{B}
    \sim\mathcal {B}$
            \State Optimize policy and Q-network with respect to $\phi$ and $\theta$ using SAC
        \EndFor
    \EndFor
\end{algorithmic}
\end{algorithm}

\section{Effects of biased reward model}
Previous work empirically showed the detrimental impact of noisy preference on the reward model \citep{lee2021b}. To further demonstrate the effects of a biased reward model, we introduce the following theorem and give the proof as follows.

\begin{assumption}[Fitting error of reward model]
    Post the phase of reward learning, the fitting error between the learned reward model $\hat{r}_\psi$ and the ground-truth reward function $r^*$ within the state-action distribution encountered by policy $\pi$ is upper-bounded by a value $\delta$:
    \begin{align}
        \expec_{(\state,\action)\sim \rho^\pi}\left| 
        \hat{r}_\psi(\state,\action)-r^*(\state,\action)
        \right| \leq \delta
    \end{align}
\end{assumption}

\begin{theorem}[Upper bound of Q-function error]
Consider a Markov Decision Process characterized by the state transition function $P$, ground-truth reward function $r^*$, and discount factor $\gamma$. Let $Q^\pi_\psi$ and $Q^\pi_*$ denote the Q-function for policy $\pi$ with respect to the learned reward model $\hat{r}_\psi$ and ground-truth reward function $r^*$, respectively. Then the error of policy estimation between $Q^\pi_\psi$ and $Q^\pi_*$ is upper-bounded by the fitting error $\delta$ of the reward model:
    \begin{align}
        \expec_{(\state,\action)\sim \rho^\pi}\left|
        Q^\pi_\psi(\state,\action)-Q^\pi_*(\state,\action)
        \right| \leq \frac{\delta}{1-\gamma}
    \end{align}
\end{theorem}
\begin{proof}
    \begin{align}
        & \expec_{(\state,\action)\sim \rho^\pi}\left|
        Q^\pi_\psi(\state,\action)-Q^\pi_*(\state,\action)
        \right| \nonumber \\
        =&\expec_{(\state,\action)\sim \rho^\pi}\left|\hat{r}_\psi(\state,\action)+\gamma\sum_{\state'}P(\state'|\state,\action)\sum_{\action'}\pi(\action'|\state')Q^\pi_\psi(\state',\action')
        -r^*(\state,\action)-\gamma\sum_{\state'}P(\state'|\state,\action)\sum_{\action'}\pi(\action'|\state')Q^\pi_*(\state',\action')
        \right| \nonumber \\
        \leq & \expec_{(\state,\action)\sim \rho^\pi}\left|\hat{r}_\psi(\state,\action)-r^*(\state,\action)\right|+\gamma\expec_{(\state,\action)\sim \rho^\pi}\sum_{\state'}P(\state'|\state,\action)\sum_{\action'}\pi(\action'|\state')\left|Q^\pi_\psi(\state',\action')-Q^\pi_*(\state',\action')\right| \nonumber \\
        = & \delta+\gamma\expec_{(\state',\action')\sim\rho^\pi}\left|Q^\pi_\psi(\state',\action')-Q^\pi_*(\state',\action')\right| \nonumber \\
        \leq & \delta + \gamma\delta + \gamma^2\delta + \gamma^3\delta + \dots \nonumber \\
        \leq & \frac{\delta}{1-\gamma}
    \end{align}
\end{proof}

\section{Proof for Theorem \ref{the:KL}} \label{app:proof_theorem_1}
\setcounter{theorem}{0}
\begin{theoremm}[\ref{the:KL}]
\textit{Consider a preference dataset \(\{(\sigma^0_i,\sigma^1_i,\tilde{y}_i)\}_{i=1}^n\), where \(\tilde{y}_i\) is the annotated label for the segment pair \((\sigma^0_i,\sigma^1_i)\) with the ground truth label \(y_i\). Let \( x_i \) denote the tuple \( (\sigma^0_i, \sigma^1_i) \). Assume the cross-entropy loss \(\mathcal{L}^{\mathtt{CE}}\) for clean data (whose $\tilde{y}_i=y_i$) within this distribution is bounded by \(\rho\). Then, the KL divergence between the predicted preference \(P_\psi(x)\) and the annotated label \(\tilde{y}(x)\) for a corrupted sample \(x\) is lower-bounded as follows:}
\begin{equation}
    \KL \left( \tilde{y}(x) \Vert P_\psi(x) \right) \geq -\ln \rho + \frac{\rho}{2} + \mathcal{O}(\rho^2)
\end{equation}
\end{theoremm}
\begin{proof}
For a clean sample $(\sigma^0,\sigma^1)$ with annotated label $\tilde{y}$ and ground-truth label $y$, we have $\tilde{y}=y$. Denote the predicted label as $P_\psi$. In PbRL, the value of $y(0)$ can take one of three forms: $y(0)\in\{0,0.5,1\}$. We categorize and discuss these situations as follows:
\begin{enumerate} [leftmargin=4mm]
\item For $y(0)=0$:

Because the cross-entropy loss $\mathcal{L}^\mathtt{CE}$ for clean data is bounded by $\rho$, we can express:
\begin{align}
    \mathcal{L}^\mathtt{CE}(P_\psi,\tilde{y})=-\ln(1-P_\psi(0))\leq\rho
\end{align}

From the above, we have:
\begin{align}
    P_\psi(0)\leq1-\exp{(-\rho)}
\end{align}

Then if the label is corrupted, denoted by $\tilde{y}_c$ (\ie $\tilde{y}_c=(1,0)$ in this case), the KL divergence between the predicted label and the corrupted label is formulated as follows:
\begin{align}
    \KL (\tilde{y}_c \Vert P_\psi)=-\ln P_\psi(0)\geq -\ln(1-\exp(-\rho))
\end{align}

\item For $y(0)=1$:

The discussion parallels the $y(0)=0$ case. Hence, the KL divergence between the predicted label and the corrupted label also maintains a lower bound:
\begin{align}
    \KL ( \tilde{y}_c \Vert P_\psi)\geq -\ln(1-\exp(-\rho))
\end{align}

\item For $y(0)=0.5$:

Although this case is not under the mistake model settings~\citep{lee2021b}, the lower bound still holds in this case. Due to the bounded cross-entropy loss $\mathcal{L}^\mathtt{CE}$ for clean data, we have:
\begin{align}
    \mathcal{L}^\mathtt{CE}(P_\psi,\tilde{y})=-\frac12\ln P_\psi(0)-\frac12\ln(1-P_\psi(0))\leq\rho
    \label{eq:loss bound for 0.5}
\end{align}

Solving the inequality (\ref{eq:loss bound for 0.5}), we can get:
\begin{align}
    P_\psi(0)^2-P_\psi(0)+\exp(-2\rho)\leq 0
    \label{eq:inequality for 0.5}
\end{align}

When $\rho\geq\ln2$, the inequality (\ref{eq:inequality for 0.5}) has a solution:
\begin{align}
    1-p\leq P_\psi(0)\leq p
\end{align}
where $p=\frac{1+\sqrt{1-4\exp(-2\rho)}}{2}$.

Then if the label is corrupted, \ie $\tilde{y}_c\in\{(0,1),(1,0)\}$, the KL divergence between the predicted label and the corrupted label is formulated as follows:
\begin{align}
    \KL (\tilde{y}_c \Vert P_\psi)\geq &\min(-\ln P_\psi(0),-\ln(1-P_\psi(0))) = -\ln p
\end{align}

Construct an equation about $\rho$:
\begin{align}
    f(\rho) = p-1+\exp(-\rho)=\frac{1+\sqrt{1-4\exp(-2\rho)}}{2}-1+\exp(-\rho)
    \label{eq:comparision eq on rho}
\end{align}
where $\rho\geq\ln2$.

Denote $z=\exp(-\rho)$, Equation (\ref{eq:comparision eq on rho}) can be simplified as follows:
\begin{align}
    f(z) = z + \frac{\sqrt{1-4z^2}}{2} - \frac12
\end{align}
where $0<z\leq\frac12$.

Derivative of function $f$ with respect to $z$, we have:
\begin{align}
    f^{'}(z) = 1 - 2\sqrt{\frac{1}{\frac{1}{z^2}-4}}
\end{align}

Function $f^{'}(z)$ decreases monotonically when $z\in(0,0.5]$, is greater than 0 on the interval $(0,\frac{\sqrt{2}}{4})$, and is less than 0 on the interval $(\frac{\sqrt{2}}{4},0.5]$. Therefore, we have:
\begin{align}
    f(z) \leq \max(f(0), f(\frac12)) = 0
\end{align}

Thus, $p\leq 1-\exp(-\rho)$ when $\rho\geq\ln2$. In turn, we have:
\begin{align}
    \KL (\tilde{y}_c \Vert P_\psi) = -\ln p \geq -\ln(1-\exp(-\rho))
\end{align}
\end{enumerate}

To sum up, inequality (\ref{eq:lower bound thero}) holds for the corrupted samples:
\begin{align}
    \KL (\tilde{y}_c \Vert P_\psi) \geq -\ln(1-\exp(-\rho))
    \label{eq:lower bound thero}
\end{align}

Perform Taylor expansion of the lower bound at $\rho=0$, we can get:
\begin{align}
    \KL (\tilde{y}_c \Vert P_\psi) \geq -\ln(1-\exp(-\rho)) = -\ln \rho + \frac{\rho}{2} + \mathcal{O}(\rho^2)
\end{align}
\end{proof}

\section{Experimental Details} \label{app:experiment_detail}
\subsection{Tasks} \label{app:tasks}
The robotic manipulation tasks from Meta-world~\citep{yu2020meta} and locomotion tasks from DMControl~\citep{tassa2018deepmind,tassa2020dm_control} used in our experiments are shown in Figure \ref{fig:env examples}.
\begin{figure*}[h]
\centering
\captionsetup[subfloat]{captionskip=-8pt}
\includegraphics[width=0.3\linewidth]{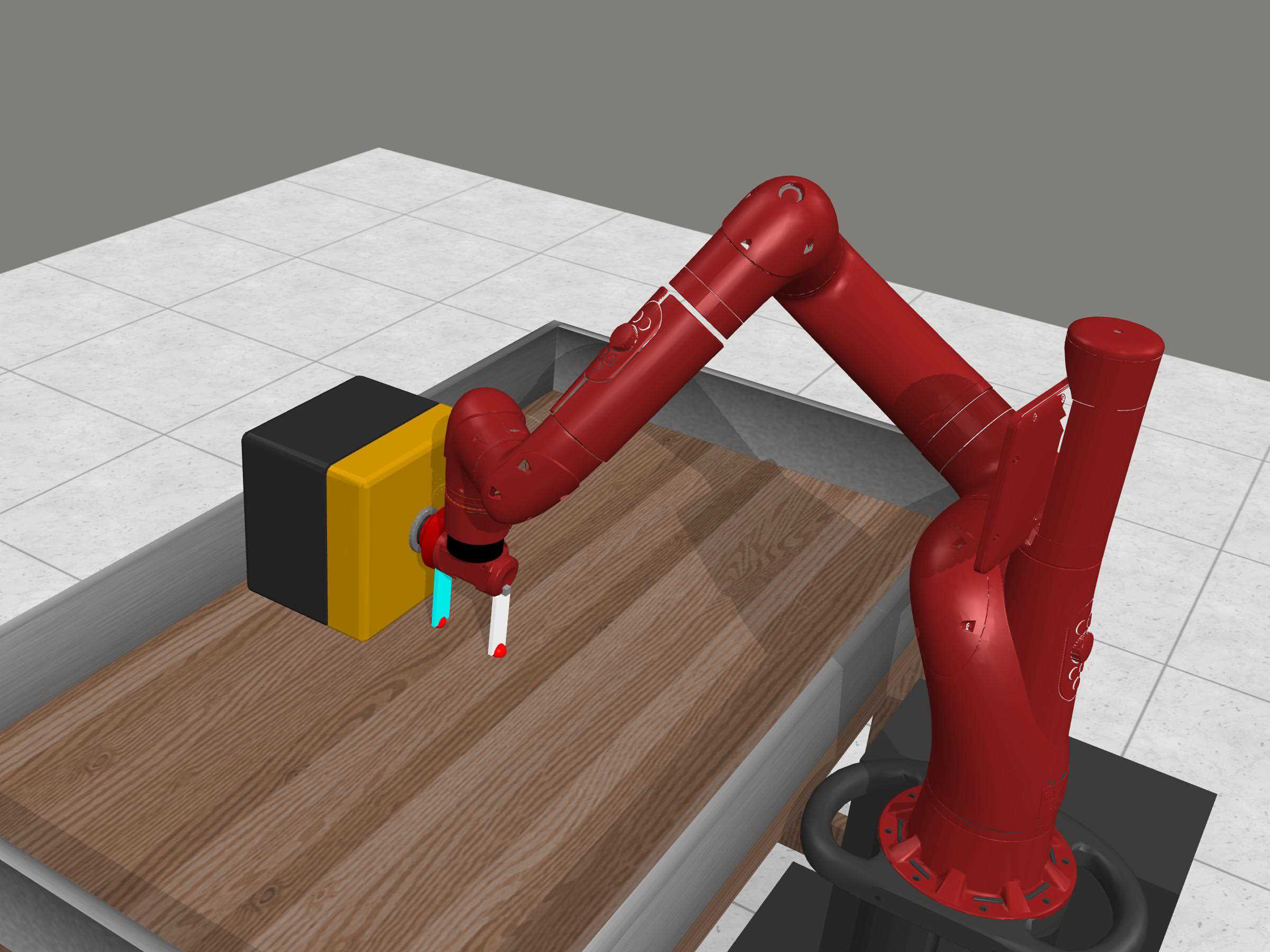}
\includegraphics[width=0.3\linewidth]{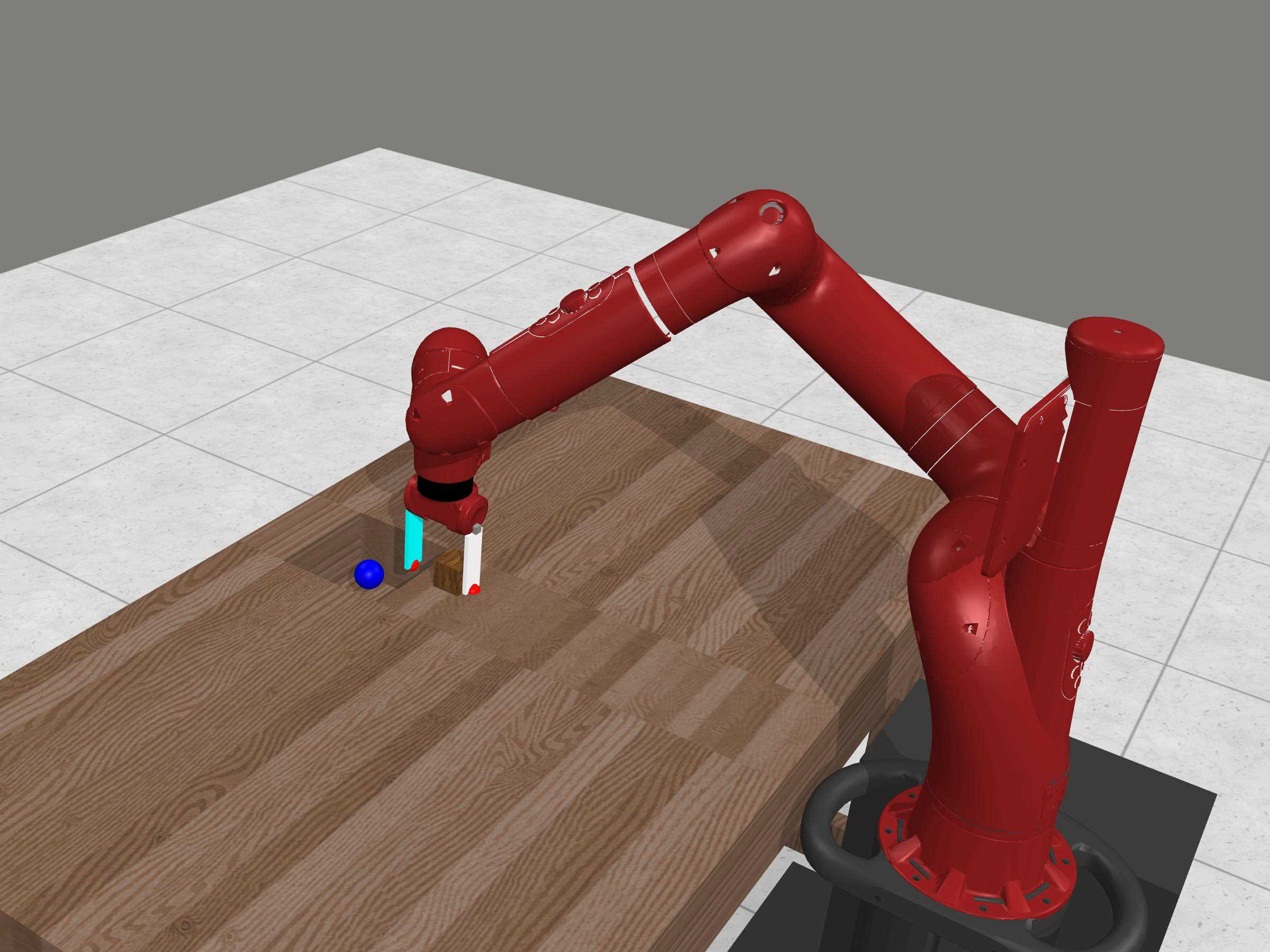}
\includegraphics[width=0.3\linewidth]{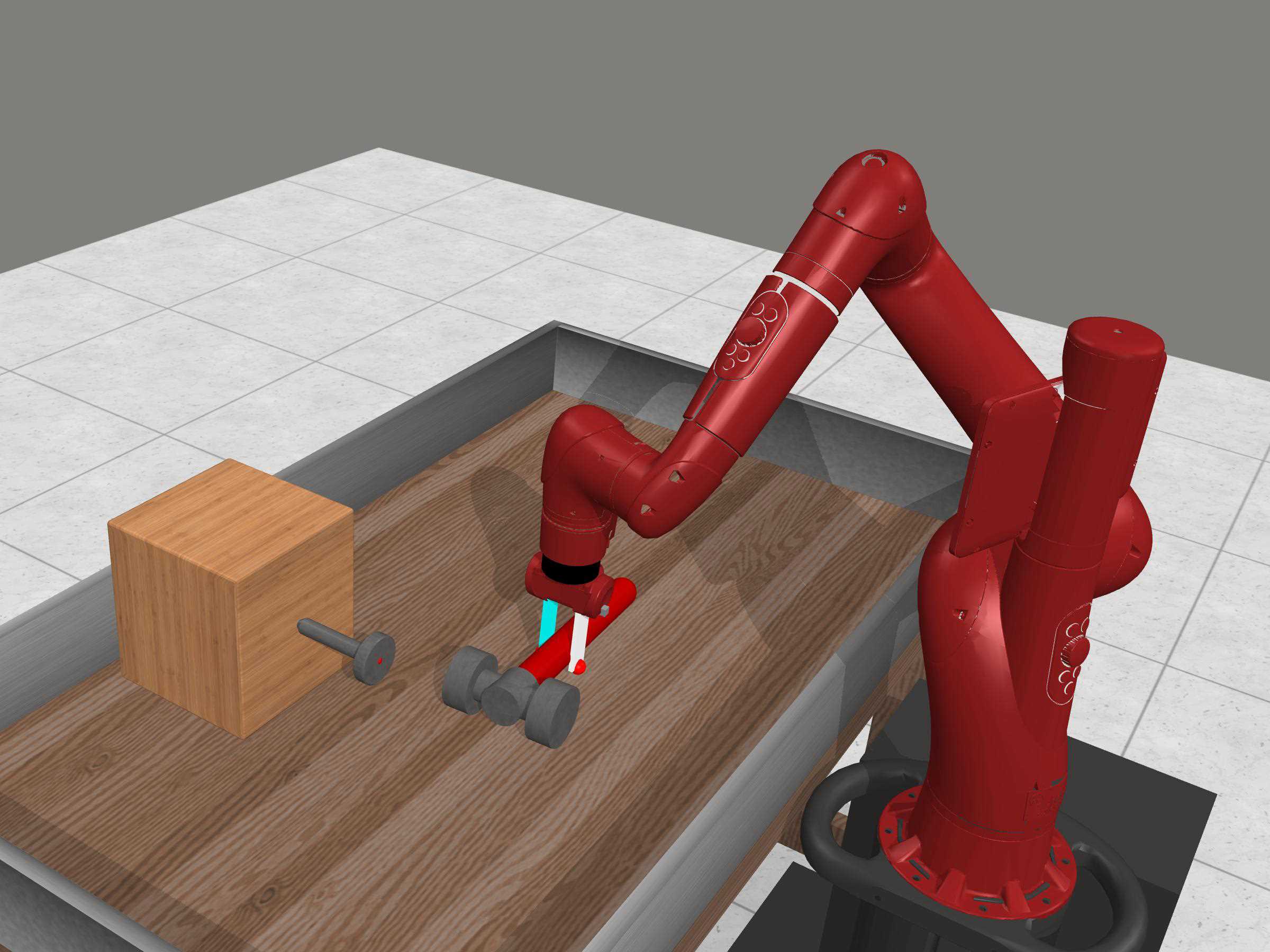}
\\
\subfloat[Button Press]{\hspace{0.33\linewidth}}
\subfloat[Sweep Into]{\hspace{0.33\linewidth}}
\subfloat[Hammer]{\hspace{0.33\linewidth}}
\\
\includegraphics[width=0.3\linewidth]{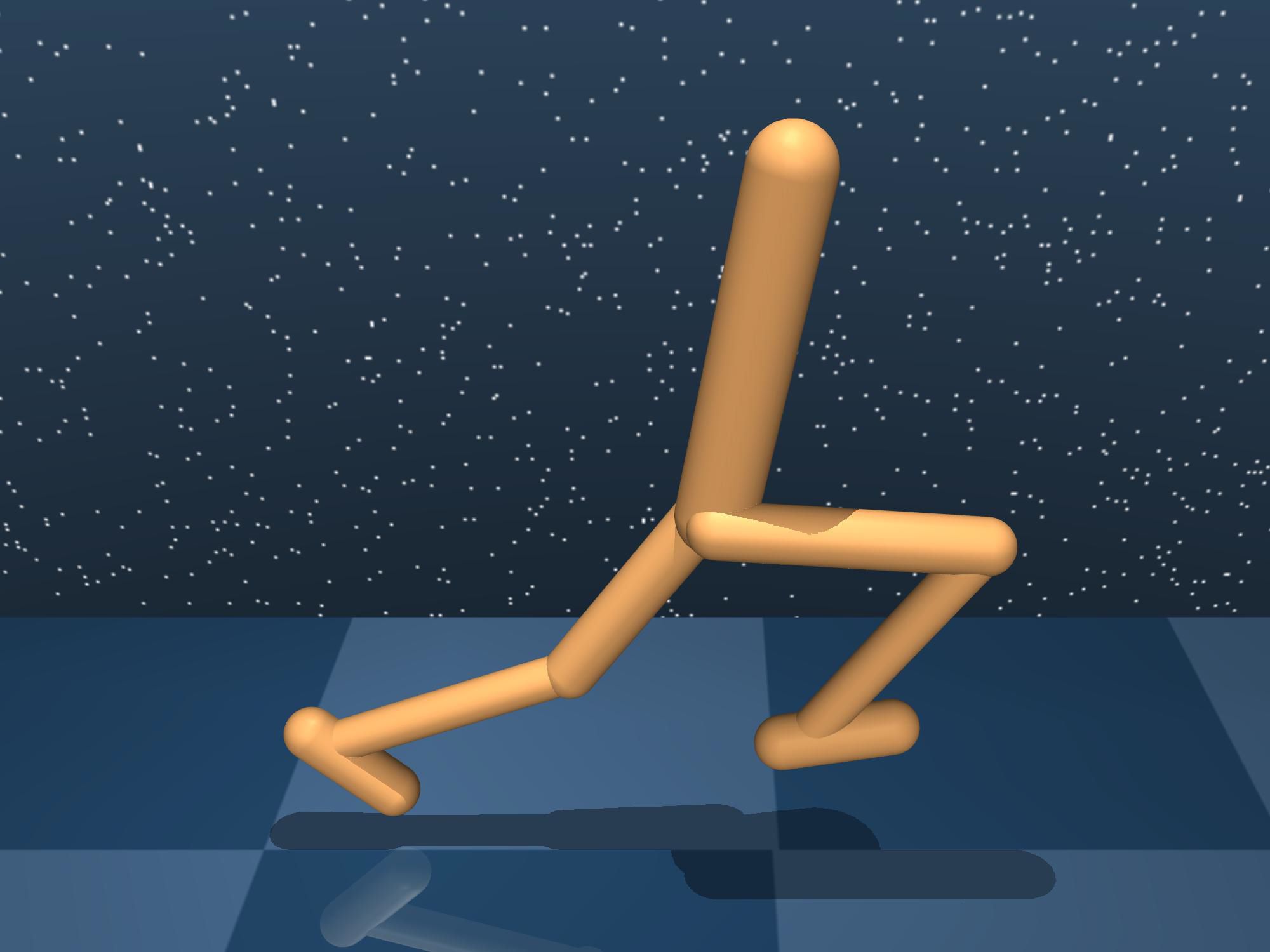}
\includegraphics[width=0.3\linewidth]{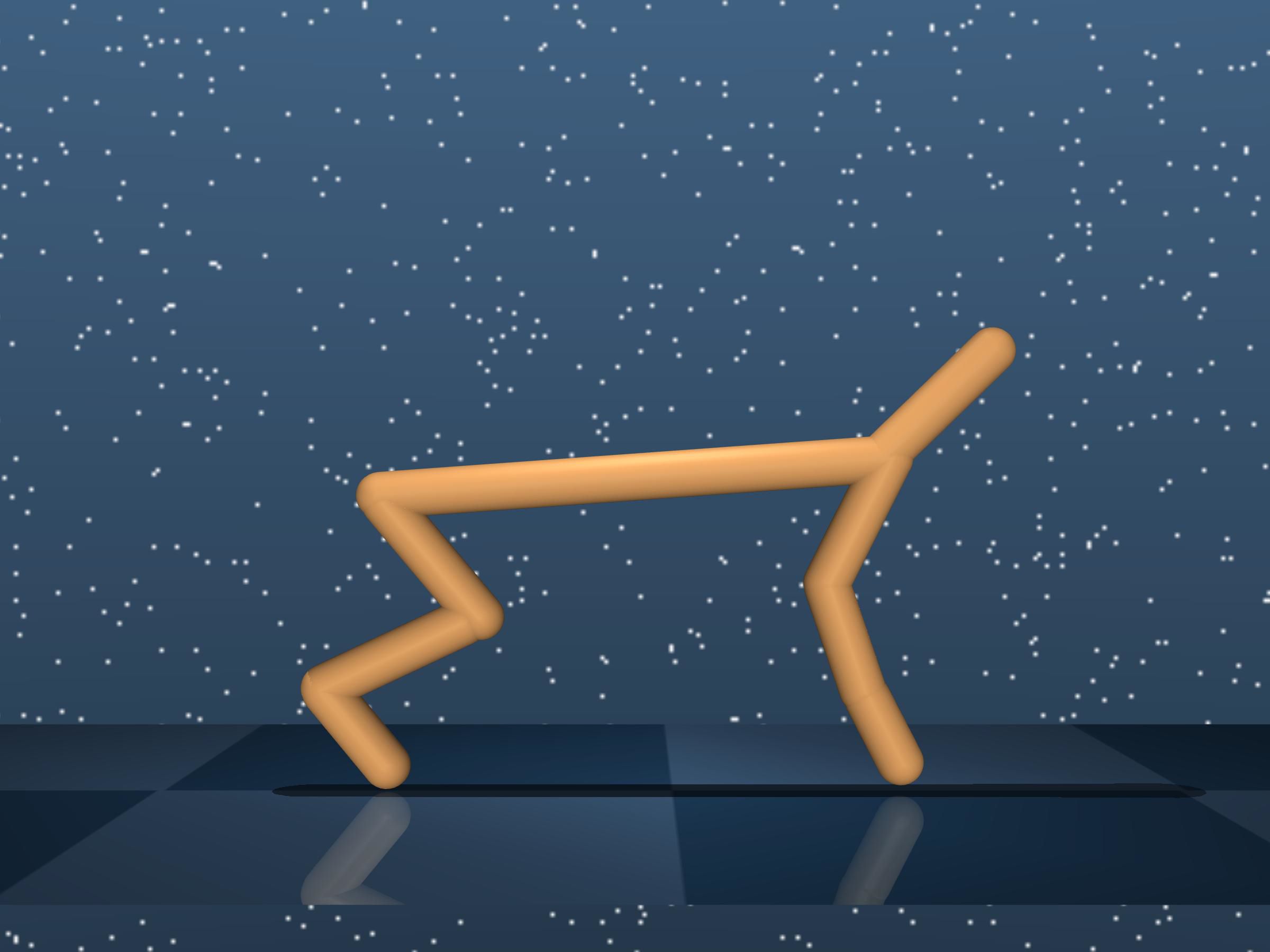}
\includegraphics[width=0.3\linewidth]{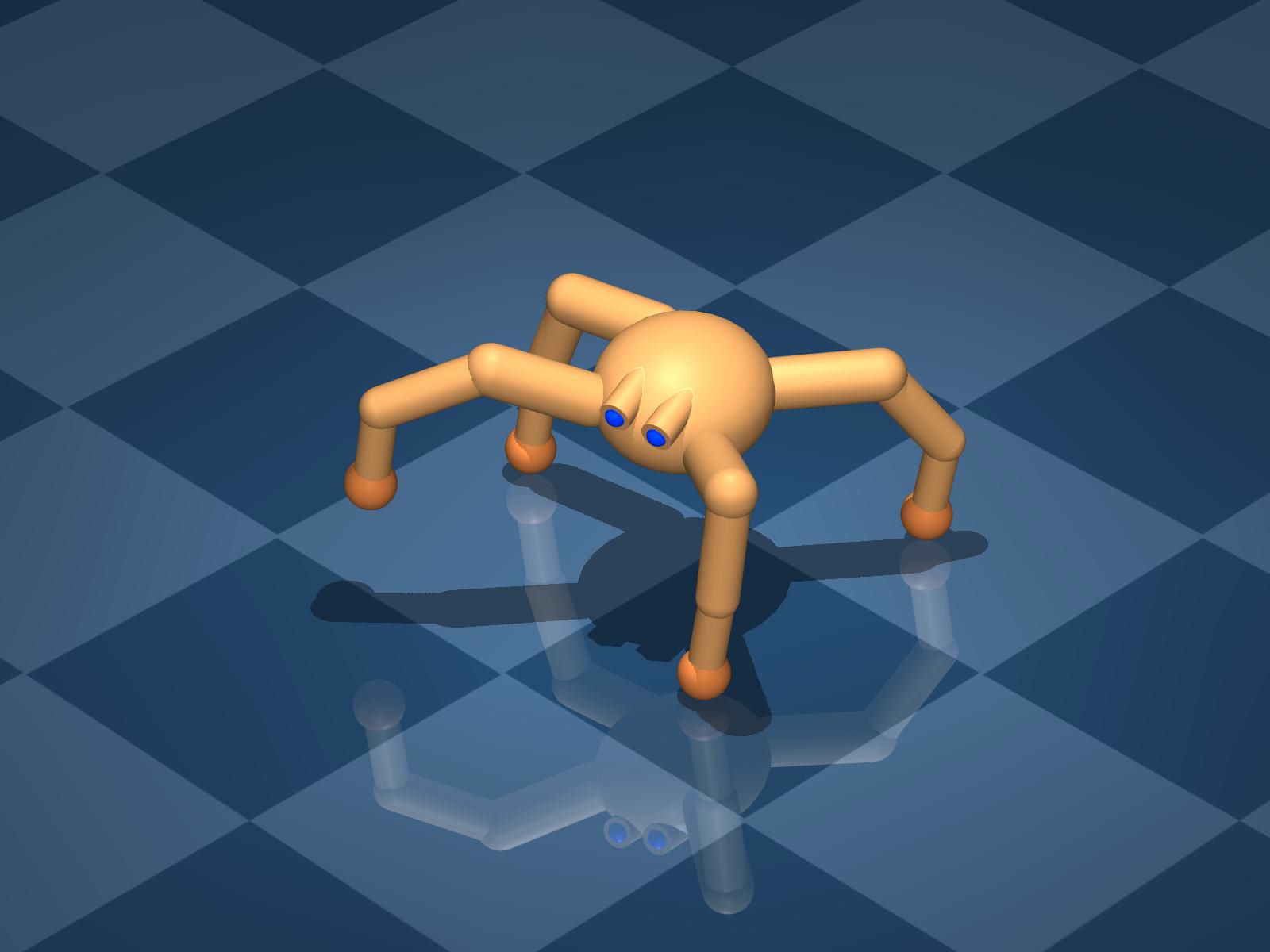}
\\
\subfloat[Walker]{\hspace{0.33\linewidth}}
\subfloat[Cheetah]{\hspace{0.33\linewidth}}
\subfloat[Quadruped]{\hspace{0.33\linewidth}}
\caption{Six tasks from Meta-world (a-c) and DMControl (d-f).}
\label{fig:env examples}
\vspace{-3mm}
\end{figure*}

\textbf{Meta-world Tasks:}
\begin{enumerate} [leftmargin=8mm]
\item[$\circ$] Button Press: An agent controls a robotic arm to press a button. The button's initial position is randomized.
\item[$\circ$] Sweep Into: An agent controls a robotic arm to sweep a ball into a hole. The ball's starting position is randomized.
\item[$\circ$] Hammer: An agent controls a robotic arm to hammer a screw into a wall. The initial positions of both the hammer and the screw are randomized.
\end{enumerate}

\textbf{DMControl Tasks:}
\begin{enumerate} [leftmargin=8mm]
\item[$\circ$] Walker: A planar walker is trained to control its body and walk on the ground.
\item[$\circ$] Cheetah: A planar biped is trained to control its body and run on the ground.
\item[$\circ$] Quadruped: A four-legged ant is trained to control its body and limbs, enabling it to crawl on the ground.
\end{enumerate}

\subsection{Implementation Details} \label{app:implemen_detail}
For the implementation of baselines, we use their corresponding publicly released repositories that are shown in Table \ref{table:source_code}. SAC serves as a performance upper bound because it uses a ground-truth reward function which is unavailable
in PbRL settings for training. The detailed hyperparameters of SAC are shown in Table \ref{table:hyperparameters_sac}. PEBBLE's settings remain consistent with its original implementation, and the specifics are detailed in Table \ref{table:hyperparameters_pebble}. For SURF, RUNE, MRN, and RIME,
most hyperparameters are the same as those of PEBBLE and other hyperparameters are
detailed in Table \ref{table:hyperparameters_surf}, \ref{table:hyperparameters_rune}, \ref{table:hyperparameters_mrn}, and \ref{table:hyperparameters_rime}, respectively. The total amount of feedback and feedback amount per session in each experimental condition are detailed in Table \ref{table:hyperparameters_condition}. The reward model comprises an ensemble of three MLPs. Each MLP consists of three layers with 256 hidden units, and the output of the reward model is constrained using the tanh activation function.

\begin{table}[H]
\small
\vspace{-10pt}
\caption{Source codes of baselines.}
\centering
\begin{tabular}{ll}
\toprule
\textbf{Algorithm} & \textbf{Url} \\
\midrule
SAC, PEBBLE & \url{https://github.com/rll-research/BPref} \\
SURF & \url{https://github.com/alinlab/SURF} \\
RUNE & \url{https://github.com/rll-research/rune} \\
MRN & \url{https://github.com/RyanLiu112/MRN} \\
\bottomrule
\end{tabular}
\label{table:source_code}
\vspace{-10pt}
\end{table}

\begin{table}[H]
\small
\vspace{-10pt}
\caption{Hyperparameters of SAC.}
\centering
\begin{tabular}{ll|ll}
\toprule
\textbf{Hyperparameter} & \textbf{Value} &
\textbf{Hyperparameter} & \textbf{Value} \\
\midrule
Number of layers & $2$ (DMControl), $3$ (Meta-world) &
Initial temperature & $0.1$ \\
Hidden units per each layer & $1024$ (DMControl), $256$ (Meta-world) &
Optimizer & Adam \\ 
Learning rate  & $0.0005$ (Walker), $0.001$ (Cheetah) & 
Critic target update freq & $2$ \\
 & $0.0001$ (Quadruped), $0.0003$ (Meta-world)& 
Critic EMA $\tau$ & $0.005$ \\ 
Batch Size & $1024$ (DMControl), $512$ (Meta-world) &
$(\beta_1,\beta_2)$  & $(0.9,0.999)$ \\
Steps of unsupervised pre-training & $9000$ & 
Discount $\gamma$ & $0.99$ \\
\bottomrule
\end{tabular}
\label{table:hyperparameters_sac}
\vspace{-10pt}
\end{table}

\begin{table}[H]
\vspace{-10pt}
\small
\caption{Hyperparameters of PEBBLE.}
\centering
\begin{tabular}{ll}
\toprule
\textbf{Hyperparameter} & \textbf{Value} \\
\midrule
Segment Length & $50$ \\
Learning rate & $0.0005$ (Walker, Cheetah), $0.0001$ (Quadruped), $0.0003$ (Meta-world) \\
Frequency of feedback & $20000$ (Walker, Cheetah), $30000$ (Quadruped), $5000$ (Meta-world) \\
Number of reward functions & $3$ \\
\bottomrule
\end{tabular}
\label{table:hyperparameters_pebble}
\vspace{-10pt}
\end{table}

\begin{table}[h!]
\vspace{-10pt}
\small
\caption{Hyperparameters of SURF.}
\centering
\begin{tabular}{ll}
\toprule
\textbf{Hyperparameter} & \textbf{Value} \\
\midrule
Unlabeled batch ratio $\mu$ & $4$ \\
Threshold $\tau$ & $0.999$ (Cheetah, Sweep Into), $0.99$ (others) \\
Loss weight $\lambda$ & 1 \\
Min/Max length of cropped segment & $45/55$ \\
Segment length before cropping & $60$ \\
\bottomrule
\end{tabular}
\label{table:hyperparameters_surf}
\vspace{-10pt}
\end{table}

\begin{table}[h!]
\small
\caption{Hyperparameters of RUNE.}
\begin{center}
\begin{tabular}{ll}
\toprule
\textbf{Hyperparameter} & \textbf{Value} \\
\midrule
Initial weight of intrinsic reward $\beta_0$ & $0.05$ \\
Decay rate $\rho$* & $0.001$ (Walker), $0.0001$ (Cheetah, Quadruped, Button Press) \\ 
 & $0.00001$ (Sweep Into, Hammer) \\
\bottomrule
\end{tabular}
\label{table:hyperparameters_rune}
\begin{minipage}{0.9\columnwidth}
*: Following the instruction of~\citet{liang2021reward}, we carefully tune the hyperparameter $\rho$ in a range of $\rho\in\{0.001,0.0001,0.00001\}$ and report the best value for each environment.
\end{minipage}
\end{center}
\vspace{-10pt}
\end{table}

\begin{table}[h!]
\vspace{-10pt}
\small
\caption{Hyperparameters of MRN.}
\centering
\begin{tabular}{ll}
\toprule
\textbf{Hyperparameter} & \textbf{Value} \\
\midrule
Bi-level updating frequency $N$ & $5000$ (Cheetah, Hammer, Button Press), $1000$ (Walker) \\
 & $3000$ (Quadruped), $10000$ (Sweep Into) \\
\bottomrule
\end{tabular}
\label{table:hyperparameters_mrn}
\vspace{-10pt}
\end{table}

\begin{table}[h!]
\vspace{-10pt}
\small
\caption{Hyperparameters of RIME.}
\centering
\begin{tabular}{ll}
\toprule
\textbf{Hyperparameter} & \textbf{Value} \\
\midrule
Coefficient $\alpha$ in the lower bound $\tau_\text{lower}$ & $0.5$ \\
Minimum weight $\beta_{\min}$ & 1 \\
Maximum weight $\beta_{\max}$ & 3 \\
Decay rate $k$ & 1/30 (DMControl), 1/300 (Meta-world) \\
Upper bound $\tau_\text{upper}$ & $3\ln(10)$ \\
$\delta$ in Equation (\ref{eq:norm intrinsic reward}) & $1 \times 10^{-8}$ \\
Steps of unsupervised pre-training & $2000$ (Cheetah), $9000$ (others) \\
\bottomrule
\end{tabular}
\label{table:hyperparameters_rime}
\vspace{-10pt}
\end{table}

\subsection{Details of scripted teachers}
\label{app:scripted teachers}
For a pair of trajectory segments $(\sigma^0, \sigma^1)$ with length $H$, where $\sigma^i=\{(\state_1^i, \action^i_1),\dots,(\state_H^i, \action_H^i)\},(i=0,1)$. The ground-truth reward function from the environment is $r(\state,\action)$. Then scripted teachers are defined as follows~\citep{lee2021b}:

\textbf{Oracle:} Oracle teacher provides ground-truth preferences. It prefers the segment with larger cumulative ground-truth rewards. For example, if $\sum_{i=1}^Hr(\state^0_i,\action^0_i)>\sum_{i=1}^Hr(\state^1_i,\action^1_i)$, then it returns the label as $(1,0)$.

\textbf{Equal:} Equal teacher gives equal preference $(0.5,0.5)$ if the difference between the cumulative rewards of two segments is small. In particular, if $|\sum_{t=1}^Hr(\state_t^1,\action_t^1)-\sum_{t=1}^Hr(\state_t^0,\action_t^0)|<\delta$, then it return the label $(0.5,0.5)$. $\delta=\frac{H}{T}R_\text{avg}\epsilon_\text{adapt}$, where T is the episode length, $R_\text{avg}$ is the average returns of current policy, $\epsilon_\text{adapt}$ is a hyperparameter and is set to 0.1 in experiments, following the setting in \citet{lee2021b}.

\textbf{Skip:} Skip teacher skips the query if the cumulative rewards of segments are small. In particular, if $\max_{i\in\{0,1\}}\sum_{t=1}^Hr(\state_t^i,\action_t^i)<\delta$, then it will skip this query. 

\textbf{Myopic:} Myopic teacher focuses more on the behavior at the end of segments. It prefers the segment with larger discounted cumulative ground-truth rewards. For example, if $\sum_{i=1}^H\gamma^{H-t}r(\state^0_i,\action^0_i)>\sum_{i=1}^H\gamma^{H-t}r(\state^1_i,\action^1_i)$, then it returns the label $(1,0)$.

\textbf{Mistake:} Mistake teacher flips the ground-truth preference labels with a probability of $\epsilon$.

\subsection{Details of experiments with human teachers}
\label{app:annotation protocol}
Human experiments adopt an online paradigm consistent with PEBBLE's pipeline, where agent training alternates with reward model training. When it is the timestep to collect preferences (post-agent training and pre-reward training), the training program generates segment pairs, saving each segment in GIF format. The program then pauses, awaiting the input of human preferences. At this juncture, human labelers engage, reading the paired segments in GIF format and labeling their preferences. Subsequently, this annotated data is fed into the training program for reward model training.

To ensure a fair comparison between RIME and PEBBLE, for one human labeler, we start the training programs for RIME and PEBBLE simultaneously. When both training programs are waiting for inputting preferences, we collate all GIF pairs from both RIME and PEBBLE and shuffle the order. Then the human labeler starts working. Therefore, from the labeler's perspective, he/she does not know which algorithm the currently labeled segment pair comes from and just focuses on labeling according to his/her preference. The labeled data is then automatically directed to the respective training programs. We conduct this experiment parallelly on each of the five labelers, thus preferences from different users do not get mixed.

In the Hopper task, the labeling process itself requires approximately 5 minutes, not accounting for waiting time. However, due to the online annotation, labelers experience downtime while waiting for agent training and GIF pair generation. Consequently, considering all factors, the total time commitment for labeling amounts to about 20 minutes per annotator.

\section{Additional Experiment Results} \label{app:addition_result}
\textbf{Effects of hyperparameters of RIME.} We investigate how the hyperparameters of RIME affect the performance under noisy feedback settings. In Figure \ref{app:ablation_hyperparameter} we plot the learning curves of RIME with a different set of hyperparameters: (a) coefficient $\alpha$ in the lower bound $\tau_\text{lower}$: $\alpha\in\{0.3,0.4,0.5,0.6\}$, (b) maximum value of $\beta_t$: $\beta_{\max}\in \{1,3,5,10\}$, (c) decay rate $k\in \{0.01, 1/30, 0.06, 0.1\}$, and (d) upper bound of KL divergence $\tau_\text{upper}\in \{2\ln(10),3\ln(10),4\ln(10)\}$. 

For the coefficient $\alpha$ in the lower bound $\tau_\text{lower}$, we find the theoretical value $\alpha=0.5$ performs the best. The maximum weight $\beta_{\max}$ and decay rate $k$ control the weight of uncertainty term in the lower bound $\tau_\text{lower}$: $\beta_t=\max(\beta_{\min}, \beta_{\max}-kt)$. The combination of $\beta_{\max}=3$ and $k=1/30$ also performs optimally. Due to the quite limited feedback amount (1000 feedback) and training epochs for the reward model (around $150\sim200$ epochs on Walker-walk), RIME is sensitive to the weight of uncertainty term. If one tries to increase $\beta_{\max}$ to add more tolerance for trustworthy samples in early-stage, we recommend increasing the decay rate $k$ simultaneously so that the value of $\beta_t$ decays to its minimum within about $1/3$ to $1/2$ of the total epochs. For the upper bound $\tau_\text{upper}$, although we use $3\ln(10)$ for balanced performance on DMControl tasks, individually fine-tuning $\tau_\text{upper}$ can further improve the performance of RIME on the corresponding task, such as using $\tau_\text{upper}=4\ln(10)$ for Walker-walk.

\begin{figure*}[ht]
\centering
\captionsetup[subfloat]{captionskip=-8pt}
\includegraphics[width=\linewidth]{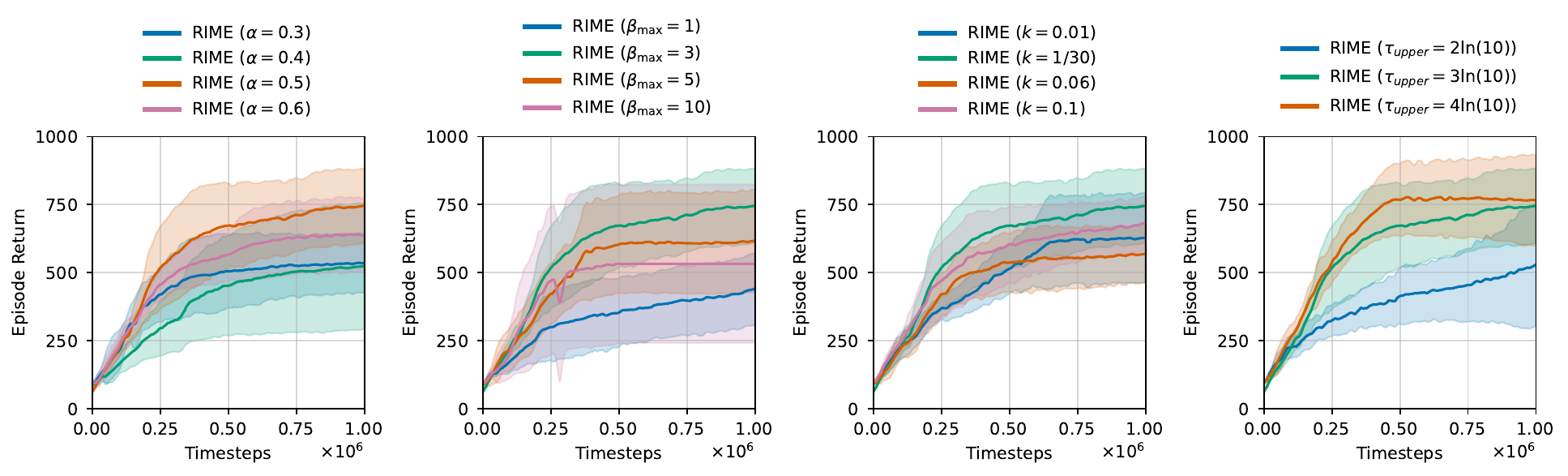}\\
\subfloat[Coefficient $\alpha$]{\hspace{0.25\linewidth}}
\subfloat[Maximum weight $\beta_{\max}$]{\hspace{0.25\linewidth}}
\subfloat[Decay rate $k$]{\hspace{0.25\linewidth}}
\subfloat[Upper bound $\tau_\text{upper}$]{\hspace{0.23\linewidth}}
\caption{Hyperparameter analysis on Walker-walk using 1000 feedback with $\epsilon=0.3$. The results show the mean and standard deviation averaged over five runs.}
\label{app:ablation_hyperparameter}
\end{figure*}

\begin{wrapfigure}{r}{0.5\textwidth}
    \centering
    \includegraphics[width=0.35\textwidth]{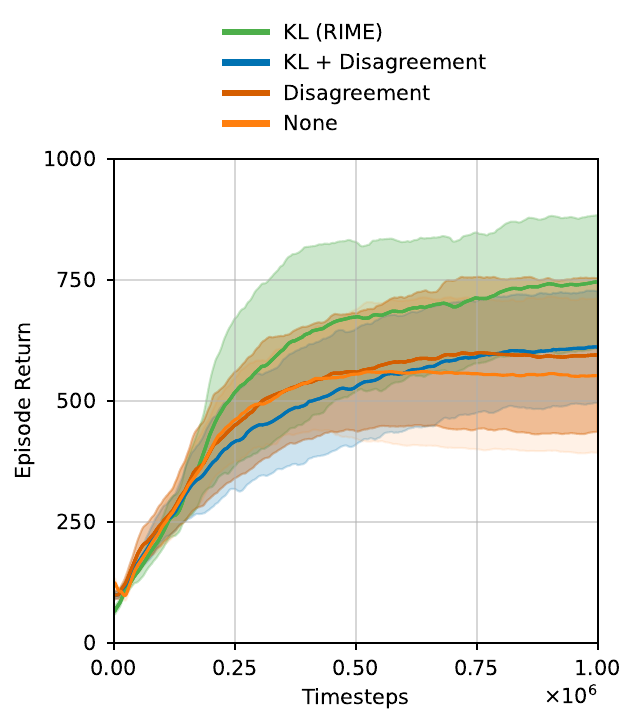}
    \caption{
    Uncertainty term analysis on Walker-walk using 1000 feedback with error rate $\epsilon=0.3$, across ten runs.
    }\label{app:ablation_uncertainty}
\end{wrapfigure}

\textbf{Effects of different uncertainty terms in the lower bound.} In RIME, we use an auxiliary uncertainty term $\tau_\text{unc}$ in the lower bound $\tau_\text{lower}$ to accommodate tolerance during the early training stages and in cases of distribution shifts. The standard deviation of the KL divergence, denoted as the KL metric in this section, is employed to discern these cases.  Here, we compare this with two other metrics: the disagreement metric and a combination of both, termed as KL + disagreement. The disagreement metric uses the standard deviation of $P_\psi[\sigma^0\succ\sigma^1]$ across the ensemble of reward models (denoted as $s_P$) to discern cases of distribution shifts: $\tau_\text{unc}=\gamma_t\cdot s_{P}$. Our intuition is that the predictions of the model for OOD data typically vary greatly. Notably, this metric induces sample-level, rather than buffer-level, thresholds, potentially offering more nuanced threshold control. The combined metric, KL + disagreement, integrates both as $\tau_\text{unc}=\beta_t\cdot s_\text{KL} + \gamma_t\cdot s_{P}$.

For reference, we also include a group devoid of any uncertainty term, termed the ``None" group. As shown in Figure \ref{app:ablation_uncertainty}, the KL metric outperforms the other approaches on Walker-walk with an error rate of $\epsilon=0.3$. This might be because the disagreement metric fluctuates violently at every query time, often leading to excessive trust in new data, which hinders the stabilization of the lower bound.

\textbf{Comparison with fixed lower bound.} We conducted experiments to compare fixed lower bound with dynamic lower bound (RIME) and presented the results in Table \ref{table:fixed lower bound}. Notice that the convergence value of $\tau_\text{lower}$ in RIME are 0.972 and 0.711 for Walker and Button-press, respectively. Table \ref{table:fixed lower bound} indicates that the dynamic lower bound employed by RIME outperforms the fixed lower bound method substantially. This superiority stems from RIME's ability to adapt its lower bound value according to the situation during training. By contrast, employing a fixed lower bound might exacerbate incorrect selections either in the early or late phase of training, depending on whether the lower bound is small or large respectively. The issue of incorrect selection will in turn lead to cumulative errors and compromise the effectiveness of selection-based robust training methods.

\begin{table}[ht]
    \small
    \begin{center}
    \caption{Ablation study of the lower bound on Walker and Button-press over 5 runs.}
\begin{tabular}{cccc}
\specialrule{.1em}{.05em}{.05em}
Method & Value of $\tau_\text{lower}$ & Walker ($\epsilon=0.3$) & Button-press ($\epsilon=0.3$) \\  \hline
 & 0.5 & 98 \stdv{92} & 36.8 \stdv{45.4} \\
Fixed & 0.7 & 179 \stdv{165} & 64.8 \stdv{26.4} \\
lower & 0.9 & 256 \stdv{88} & 58.0 \stdv{41.9} \\
bound & 1.1 & 468 \stdv{186} & 49.0 \stdv{37.6} \\
 & 1.3 & 325 \stdv{79} & 56.6 \stdv{34.0} \\
\rowcolor{mygray} RIME & dynamic & \textbf{741} \stdv{139} & \textbf{80.0} \stdv{27.7} \\
\specialrule{.1em}{.05em}{.05em}
\end{tabular}
\label{table:fixed lower bound}
\end{center}
\end{table}

\textbf{Explore the error rate limits of RIME.} To understand the boundaries of RIME, we conduct supplementary experiments in both the Walker-walk and Cheetah-run environments, varying the error rates $\epsilon$ within the range $\{35\%, 40\%, 45\%, 50\%\}$, across 5 runs. The results, presented in Table \ref{table:error rate limits}, reveal that RIME failed with 45\% and 40\% noisy data in the Walker-walk and Cheetah-run environments, respectively. Interestingly, even with the feedback amount increased to ten times the minimum, it has not change the result of failure. The performance gains from increasing the amount of feedback are limited under high error rates ($\epsilon\geq 0.4$).

\begin{table}[t]
    \small
    \begin{center}
    \caption{Performance of RIME with different noise levels over 5 runs.}
\begin{tabular}{l|l|cccc}
\specialrule{.1em}{.05em}{.05em}
\multirow{2}{*}{Environment} & Feedback & \multicolumn{4}{c}{Error rate} \\ \cline{3-6}
 & volume & 35\% & 40\% & 45\% & 50\% \\  \hline
\multirow{3}{*}{Walker} & 1000 & 646.58 & 497.63 & 164.31 & / \\
 & 5000 & / & 500.64 & 210.86 & 145.05 \\
 & 10000 & / & 554.69 & 312.61 & 217.34 \\ \hline
\multirow{3}{*}{Cheetah} & 1000 & 403.72 & 246.51 & / & / \\
 & 5000 & 440.4 & 347.4 & / & / \\
 & 10000 & 503.57 & 393.0 & / & / \\
\specialrule{.1em}{.05em}{.05em}
\end{tabular}
\label{table:error rate limits}
\end{center}
\end{table}

\textbf{Trade-off between sample efficiency and robustness.} We repeat the same analysis that is detailed in \ref{subsec:ablation study} for PEBBLE and present the results in Table \ref{table:trade off error rate for pebble} and Table \ref{table:trade off amount for pebble} below. Similarly, we observe a gradual decline in PEBBLE's performance with rising error rates. Doubling the amount of feedback engenders a marginal enhancement to PEBBLE; however, this improvement is negligible on Quadruped and Button-press when $\epsilon\geq0.2$. Intriguingly, by comparing the 4-th column ($N$) of Table \ref{table:trade off amount} with the 5-th column (2$N$) of Table \ref{table:trade off amount for pebble}, we find that RIME even outperforms PEBBLE with only half the number of feedbacks in most cases.

\begin{table}[b]
    \small
    \begin{center}
    \caption{Results of PEBBLE as the error rate increases with constant amount of feedback, across 5 runs.}
    \begin{tabular}{l|c|ccc}
    \specialrule{.1em}{.05em}{.05em}
        \multirow{2}{*}{Environment} & Feedback & \multicolumn{3}{c}{Error rate} \\ \cline{3-5}
         & volume & 0.1 & 0.2 & 0.3 \\
        \hline
        Walker & 500 & \textbf{749} \stdv{123} & 490 \stdv{252} & 230 \stdv{172} \\
        Quadruped & 2000 & \textbf{292} \stdv{166} & 171 \stdv{26} & 125 \stdv{38} \\
        Button-press & 10000 & \textbf{93.1} \stdv{10.6} & 21.6 \stdv{15.3} & 17.8 \stdv{25.2} \\
        Hammer & 20000 & \textbf{36.6} \stdv{41.4} & 20.0 \stdv{17.8} & 15.7 \stdv{12.0} \\
        \specialrule{.1em}{.05em}{.05em}
    \end{tabular}
    \label{table:trade off error rate for pebble}
    \end{center}
\end{table}

\begin{table}[b!]
    \small
    \begin{center}
    \caption{Results of PEBBLE as the feedback volume increases with constant error rate, across 5 runs. $N$ refers to the minimal feedback volume for each environment shown in Table \ref{table:hyperparameters_condition}.}
    \begin{tabular}{l|c|c|cc}
    \specialrule{.1em}{.05em}{.05em}
        \multirow{2}{*}{Domain} & \multirow{2}{*}{Environment} & \multirow{2}{*}{Error rate} & \multicolumn{2}{c}{Feedback volume} \\ \cline{4-5}
         & & & $N$ & 2$N$ \\
        \hline
        \multirow{4}{*}{DMControl} & \multirow{2}{*}{Walker} & 0.2 & 490 \stdv{252} & \textbf{656} \stdv{158} \\
         & & 0.3 & 230 \stdv{172} & \textbf{431} \stdv{157} \\ \cline{2-5}
         & \multirow{2}{*}{Quadruped} & 0.2 & 171 \stdv{26} & \textbf{212} \stdv{47} \\
         & & 0.3 & 125 \stdv{38} & \textbf{165} \stdv{35} \\ \hline
        \multirow{2}{*}{Meta-world} & \multirow{2}{*}{Button-press} & 0.2 & 21.6 \stdv{15.3} & \textbf{26.2} \stdv{35.7} \\
         & & 0.3 & 17.8 \stdv{25.2} & \textbf{22.0} \stdv{13.8} \\
        \specialrule{.1em}{.05em}{.05em}
    \end{tabular}
    \label{table:trade off amount for pebble}
    \end{center}
\end{table}

\end{document}